%% file: main.tex
\documentclass{article}

\usepackage{microtype}
\usepackage{graphicx}
\usepackage{subcaption}
\usepackage{booktabs} 

\usepackage{hyperref}

\input{macros.tex}

\usepackage{dsfont}
\usepackage[preprint]{icml2026}

\usepackage{amsmath}
\usepackage{amssymb}
\usepackage{mathtools}
\usepackage{amsthm}

\usepackage[capitalize,noabbrev]{cleveref}

\theoremstyle{plain}
\newtheorem{theorem}{Theorem}[section]
\newtheorem{proposition}[theorem]{Proposition}

\newtheorem{corollary}[theorem]{Corollary}
\theoremstyle{definition}

\newtheorem{assumption}[theorem]{Assumption}
\theoremstyle{remark}
\newtheorem{remark}[theorem]{Remark}

\usepackage[textsize=tiny]{todonotes}

\icmltitlerunning{From Collapse to Improvement}

\begin{document}

\twocolumn[
  \icmltitle{From Collapse to Improvement:  Statistical Perspectives on the Evolutionary Dynamics of Iterative Training on Contaminated Sources}



  \icmlsetsymbol{equal}{*}

  \begin{icmlauthorlist}
    \icmlauthor{Soham Bakshi}{equal,yyy}
    \icmlauthor{Sunrit Chakraborty}{equal,comp}
  \end{icmlauthorlist}

  \icmlaffiliation{yyy}{Department of Statistics, University of Michigan, Ann Arbor}
  \icmlaffiliation{comp}{Department of Statistical Science, Duke University}

  \icmlcorrespondingauthor{Soham Bakshi}{bakso@umich.edu}

  \icmlkeywords{model collapse, }

  \vskip 0.3in
]



\printAffiliationsAndNotice{}  

\begin{abstract}
  The problem of model collapse has presented new challenges in iterative training of generative models, where such training with synthetic data leads to an overall degradation of performance. This paper looks at the problem from a statistical viewpoint, illustrating that one can actually hope for improvement when models are trained on data contaminated with synthetic samples, as long as there is some amount of fresh information from the true target distribution. In particular, we consider iterative training on samples sourced from a mixture of the true target and synthetic distributions. We analyze the entire iterative evolution in a next-token prediction language model, capturing how the interplay between the mixture weights and the sample size controls the overall long-term performance. With non-trivial mixture weight of the true distribution, even if it decays over time, simply training the model in a contamination-agnostic manner with appropriate sample sizes can avoid collapse and even recover the true target distribution under certain conditions. Simulation studies support our findings and also show that such behavior is more general for other classes of models.
\end{abstract}

\section{Introduction}
Generative AI models, especially large language models (LLMs) based on transformers \cite{vaswani2017attention}, such as ChatGPT \cite{radford2018improving, radford2019language, brown2020language}, have revolutionized tasks like text generation, summarization, sentiment analysis, and text-to-image/audio translation. Other deep generative models such as diffusion models \cite{ho2020denoising, song2020score}, flow-based models \cite{dinh2014nice, dinh2016density, kingma2018glow}, GANs \cite{goodfellow2014generative, radford2015unsupervised, arjovsky2017wasserstein}, and VAEs \cite{kingma2013auto} have similarly advanced data synthesis by learning complex distributions from massive datasets. However, these models are highly data-hungry; for example, GPT-3 has 175 billion parameters and was trained on 570GB of text.

As high-quality, human-generated data becomes scarce \cite{chang2024survey, bender2019data}, models increasingly rely on synthetic data generated by earlier versions. With AI-generated content now pervasive online, e.g., LLMs contributed to ~5\% of Wikipedia \cite{brooks2024rise}, which was then used in training ChatGPT-4 \cite{del2023large}—training datasets are becoming increasingly self-referential. Similar trends are seen in vision datasets \cite{schuhmann2022laion}. This feedback loop of iterative training on synthetic data, whether intentional or not, departs from classical i.i.d. assumptions and raises critical questions about the long-term behavior and evolution of such models.

Researchers pursuing this issue soon came across the phenomenon of \emph{model collapse}~\cite{shumailov2024ai}, a compounding degenerative effect in the performance of models when recursively re-trained on synthetic data. This self-consuming cycle (also called \emph{autophagous loop} \cite{alemohammad2023self}) of retraining on self-generated synthetic data gradually drives the model far from the true data-generating distribution. While this extreme case may not be true in reality, there is a gap in understanding whether combining synthetically generated data along with fresh real data can improve the statistical estimation of the underlying true distribution over time. We investigate this fundamental question and provide theoretical insights into the long-term evolution of models under iterative training. Although a few recent works~\cite{iterativestability,fu2024theoreticalunderstanding,LLMsyntheticdata} have examined models trained on real-synthetic mixtures, theoretical understanding of the overall training evolution remains limited. In this paper, we propose a rigorous statistical framework to study \emph{evolutionary dynamics of the generative AI models iteratively trained on data sourced from a mixture of human and machine-generated content}. Our main contributions are summarized as follows
\begin{itemize}
    \item We provide theoretical analysis of iterative training process with data coming from a mixture of true distribution and previously fitted model. For a simplified setting of statistical language models, we derive the exact trajectory of this evolution and precisely characterize conditions on mixture weights and sample sizes under which statistical estimation of the target distribution improves.  
    \item We show more generally that a single-step of iterative training cannot improve estimation if there is no fresh information from the true data distribution. On the other hand, if there is, then statistical estimation can improve under mild regularity conditions.
    \item We conduct extensive simulations to demonstrate our findings and that the evolutionary behavior studied in the statistical language model extends to more general cases.
\end{itemize}

The rest of the paper is organized as follows. In Section \ref{sec:framework}, we introduce our theoretical framework and review the existing literature. In Section \ref{sec:language model}, we look at simple next-token prediction language models and analyze the iterative evolution in depth. In Section \ref{sec:one step}, we generalize our insights for a single step of this iterative training by asking the question - when is it possible to improve the estimator statistically.  Section \ref{sec:simulation} provide extensive simulations to support the theoretical insights and demonstrate the generality of our results. 
\section{Theoretical Setup}
\label{sec:framework}
We begin our discussion with a general mathematical formulation of this iterative training process. We illustrate the formulation through LLMs like ChatGPT for natural language generation as the example. Assume that there is a fixed language model which governs the way humans generate text documents and treat this as the ground-truth data distribution. The goal of a language model (e.g. ChatGPT) is to learn this true data distribution through samples of text documents with the aim of generating from that distribution. When the first version of GPT was trained, the source of training data was mostly online corpus consisting of predominantly samples from the true data distribution or human language model. However, while training ChatGPT3, the online corpus consisted of texts from different sources — some of it might already be present when training ChatGPT2, while the new training data (i.e., not used previously) could either be written by human (\emph{real} data) or generated by a previous version of GPT (\emph{synthetic} data). However, during the training process, the source of each sample in the training data is unknown, and as such, the entire dataset is trained using the model in an agnostic manner. The question is whether such training improves the performance of the model over time statistically, in the sense of approximating the true human language model.

\begin{figure}
    \centering
    \includegraphics[trim={1.6cm 5.4cm 1.6cm 4.6cm}, clip, width=1.0\linewidth]{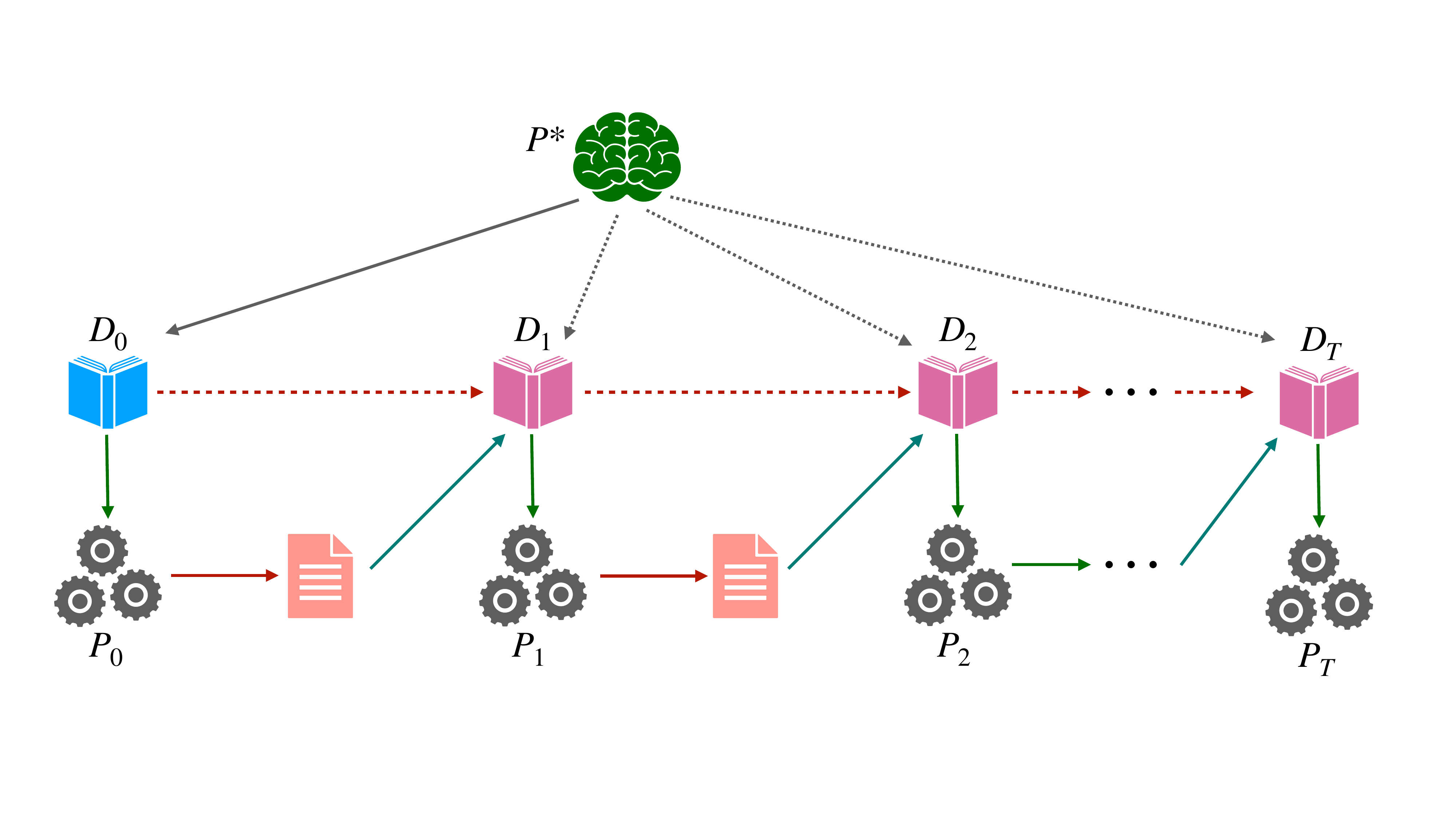}
    \caption{$P^*$ is the true underlying true data distribution, at time $t$, model $M_t$ is trained on data $D_t$, where $D_0$ is purely a sample from $P^*$, but subsequently $D_t$ consists of (i) synthetic part, (ii) accumulated part, (iii) fresh information. Can $P_t$ approximate $P^*$ well? In parametric setting, $P_t=P_{\hat{\theta}_t}$.}
    \label{fig:recursive_training}
\end{figure}

\paragraph{Mixture modeling:} Mathematically, let $P^*$ denote the true data distribution  and $\cP=\{P_{\theta}: \theta\in\Theta\}$ be the class of models. At time $t=0,1,\dots$, consider a data corpus $D_t\sim P_t$ of size $|D_t|=n_t$ based on which the model is trained and we denote the fitted model as $\hat{P}_t := P_{\hat{\theta}_t}$, where $\hat{\theta}_t$ is the estimator based on $D_t$ obtained by optimizing some loss function (e.g., it could be $L_2$-loss for regression, or logistic loss or a more general negative log-likelihood loss). At the first stage $P_0=P^*$, i.e. the data corpus $D_0$ consists of samples only from the true data distribution. Subsequently, at $t>0$, the training data $D_t$ consists of samples from mainly two different source types (i) \textit{fresh} samples from true data distribution $P^*$ (which we refer to as \emph{real data}), and (ii) \emph{synthetic data} generated from previously fitted models $\{\hat{P}_{s} : s<t\}$. Another potential source is \textit{reusing data from prior times} (accumulation), see Section \ref{sec:accu}, however our focus is primarily on the two main sources. For simplicity, assuming that only the latest version of the model is used for the synthetic generation, then the data generative mechanism at the $t+1$ time point can be modeled as a mixture \begin{equation}
 P_{t+1} = \alpha_{t+1} P^* + (1-\alpha_{t+1})\hat{P}_t   
\end{equation} where $\alpha_t$ captures the proportion of \emph{fresh} real data in the $t$-th stage data $D_t$. The complete training process is illustrated in Figure~\ref{fig:recursive_training}. Note that even through the data arises from different heterogeneous sources, the model is still fitted using the same optimization method with the fixed model class $\cP$. It is also important to recognize that the mixture assumption is at the population level -- this is different from an empirical mixture involving adding some of the real samples from the first time point to potentially all synthetic samples generated at time $t$.

The goal of this study is to understand under what conditions the iterative training with mixture of real and synthetic data can actually improve the models over time. Precisely, our interest is to examine how the expected distance of the fitted model from the true distribution, $\mathbb{E}d(P^*,\hat{P}_t)$ evolves with training iterations $t\geq0$, where $d(.,.)$ is some appropriate distance measure. We specifically explore the conditions on the mixture weights $(\at)_{t\geq 1}$ and sample sizes $(n_t)_{t\geq0}$ such that not only collapse is avoided, but also models eventually improve with training iterations, i.e. $\mathbb{E}d(P^*,\hat{P}_t)\lesssim \mathbb{E}d(P^*,\hat{P}_0)$ and ideally achieves limiting consistency, i.e. $\mathbb{E}d(P^*,\hat{P}_t)\to 0$ with growing sample sizes $n_t \uparrow \infty$ as $t\to\infty$.

\paragraph{Related Work:}
The problem of model collapse was first identified and studied in \cite{shumailov2024ai} and was empirically validated by \cite{alemohammad2023self,bohacek2023nepotistically, martinez2023combining,martinez2023towards,guo2024curiousdeclinelinguisticdiversity} for a wide variety of models and web-scale datasets. However, there are relatively fewer works that theoretically study the effect of model collapse. Theoretical analysis of model collapse was done in \cite{shumailov2023curse}  for Gaussian data, \cite{dohmatob2024model} in regression settings and \cite{dohmatob2024taletailsmodelcollapse} study the change in scaling laws when iterative training is done with synthetic data. 


\cite{alemohammad2023self} conducts experiments across different mixed training setups, indicating that incorporating real data can help alleviate model collapse. For simplified language models, \cite{LLMsyntheticdata} examine training with discrete mixtures of real and synthetic data, deriving upper bounds on synthetic data proportions to prevent collapse. However, note our population level mixture modeling differs from data mixtures considered in~\cite{LLMsyntheticdata}, where real data is injected at each training step from a fixed (finite) original sample. \cite{iterativestability} considers a general maximum likelihood estimation framework in the mixture setting involving a proportion of data from the true data distribution and analyze the stability of training process at the distribution level
under a locality assumption in parameter space. For diffusion models, \cite{fu2024theoreticalunderstanding} provide a  control of the $\TV$ distance between $\hat{P}_t$ and $P^*$, and also shows importance of balancing portions of real and synthetic data. Our work is in similar spirit with the above settings, but is directed towards studying statistical improvement or consistency instead of just focusing on preventing collapse. Furthermore, in contrast with most existing works, we consider asymptotic regimes with sample sizes growing at each iteration and seek to characterize the limiting behavior of models risk. 

 Finally, we also note some recent studies have also emerged as potential solutions to model collapse. A line of work~\cite{accumulation,kazdan2024accumulating} explores \emph{data accumulation} to prevent model collapse. However, we show accumulation cannot improve performance over the iterations (see Section~\ref{sec:accu}). While few recent works investigate strategies for curating or filtering synthetic data to prevent model collapse \cite{feng2024beyondcollapses,ferbach2024selfconsuminggenerativemodelscurated}, a thorough evaluation of their effectiveness remains absent.


\section{Iteratively Trained Language Models}
\label{sec:language model}

To theoretically study the evolution of the iterative training process for Large Language models, we focus on a simplified statistical language model with context-based next token prediction. This modeling approach is similar to \cite{LLMsyntheticdata} and theoretically boils down to multinomial parameter estimation with $K$ target tokens with model class $P_{\theta}=\{\text{Cat}(\theta): \theta\in\Delta^{K-1}\}$, , where $\text{Cat}$ denotes a categorical distribution and $\Delta^{k-1}=\{x\in\bbR^k : x_i\geq 0, \sum_i x_i = 1\}$ is the $(k-1)$-dimensional probability simplex

\paragraph{Background on Language Models:} Say we have a vocabulary of size $K$ and $C \leq K^{l}$ many possible contexts of maximum length $l$. Assuming that the language data is generated from some unknown conditional probabilities given the contexts, the probability of the next token being $k \in[K]:=\{1, \ldots, K\}$ given some context $j=\left(j_1, \ldots, j_{\ell}\right) \in[C]$ is denoted by $\mathbb{P}\rbracket{Y=k \mid X=j}$
where $X$ and $Y$ are discrete random variables specifying a context and the next token respectively. In practice, conditional probabilities are unknown but we have access to a (large) corpus of training dataset $\left\{\left(x_l, y_l\right)\right\}_{l \in[N]}$ of $N$ samples of contexts and next-token pairs represented by $x_l \in\left\{e_1, \ldots, e_C\right\}$ and $y_l \in\left\{e_1, \ldots, e_K\right\}$ where $e_i$'s denote the canonical vectors. Then we consider estimating the conditional probabilities via the Softmax classifier, which minimizes the categorical cross-entropy loss function:
$$
\underset{\mathbf{W}=\left[w_1, \ldots, w_K\right] \in \mathbb{R}^{C \times K}}{\arg \min }-\frac{1}{N} \sum_{l=1}^N y_l^{\top} \log \sigma\left(\mathbf{W}^{\top} x_l\right)
$$
where $\sigma(v)=\frac{\exp (v)}{\sum_{k=1}^K \exp \left(v_k\right)}$ is the Softmax function and the functions exp and log are applied entry-wise. Note that we choose to work with one-hot embeddings as representations for tractable theoretical analysis even though current state-of-the-art language models, the $x_l$'s are context representations computed using attention or transformer based networks. The estimates of conditional probability solving the Softmax problem are the empirical means:
\begin{align*}
  \mathbb{\widehat{P}}\rbracket{Y=k \mid X=j}&=\frac{\exp \left(\hat{w}_k^{\top} e_j\right)}{\sum_{i=1}^K \exp \left(\hat{w}_i^{\top} e_j\right)}=\frac{1}{\left|\mathcal{C}_j\right|} \sum_{l \in \mathcal{C}_j} y_{l k} \\
\text { with } \quad \mathcal{C}_j &=\left\{l \in[N] \mid x_l=e_j\right\}.  
\end{align*} 
Fixing a context say $j\in [C]$ (equivalently setting $x = e_j$), denote true conditional probabilities as \begin{align*}\theta^{*}(k) &= \mathbb{P}\rbracket{Y=k \mid X=j}, \ \forall k\in [K], \\
   \theta^{*} &= \begin{bmatrix}
      \theta^{*}(1) \cdots \theta^{*}(K)  
   \end{bmatrix}^{\top} \in \Delta^{K-1}.
\end{align*} Then the above classification problem boils down to a statistical problem of categorical proportion estimation with data
$\left\{\left(e_{j}, y_l\right)\right\}_{l \in \mathcal{C}_j}$. Here, the true data distribution distribution $P_{\theta^*}:=\text{Cat}(\theta^{*})$ which lies in the model class $\{\text{Cat}(\theta): \theta \in \Delta^{K-1}\}$.

\paragraph{Iterative Training Process:}
Now, let's frame the problem for iterative training setting where the next stage new data is generated from a Categorical with class probabilities as mixture of the true class probabilities and the current fitted probabilities. The first stage data $Y^{0} = \nbracket{Y^{0}_{1}, \ldots, Y^{0}_{n_{0}}}$ are $n_0$ i.i.d samples from the true data distribution $P_{\theta^{*}}=\text{Cat}(\theta^{*})$, using them we get the initial estimates as $$\widehat{\theta}_{0}(k) = \frac{\sum_{j=1}^{n_{0}}\mathds{1}[Y^{0}_{j}=k]}{n_{0}}.$$ For any training time $t\geq 0$, given the current estimate $\widehat{\theta}_{t}$, the next stage data 
$Y^{t+1} = \nbracket{Y^{t+1}_{1}, \ldots, Y^{t+1}_{n_{t+1}}}$ is generated $\iid$ from 
\begin{align*}
P_{t+1}:=&\text{Cat}(\theta_{t+1}) =\att P_{\theta^{*}}+ (1-\att)\widehat{P}_t \\=&  \text{Cat}(\alpha_{t+1}\theta^{*} + (1-\alpha_{t+1})\widehat{\theta}_{t}),
\end{align*}  which recursively yields the next stage estimate $\widehat{\theta}_{t+1}$ using the corresponding sample proportions from $Y^{t+1}$ (agnostic of data source). Our aim here is to understand the entire trajectory of $\hat{\theta}_t$ (equivalently $\widehat{P}_t=P_{\widehat{\theta}_{t}}$) in terms of $$R_t:=\mathbb{E}\Vert\hat{\theta}_t - \theta^*\Vert_2^2,$$ the risk of $t$-th estimate, measuring the quality of model fit.

\begin{theorem}
\label{thm:1}
In the above setting, the sequence $\nbracket{R_t}_{t\geq 1}$(as defined above) satisfies the following recurrence
        $$R_t = \frac{n_{0}}{n_t}R_{0} + \frac{n_t-1}{n_t} (1-\alpha_t)^2 R_{t-1}$$
        where $R_0=\frac{\sum_k \theta^*(k)(1-\theta^*(k))}{n_{0}}=\frac{1-\Vert\theta^*\Vert_2^2}{n_{0}}.$
\end{theorem}

From now on, let $R_{\infty}$ denote the limiting model risk $\lim_{t\to\infty}R_{t}$, when the limit exists (otherwise $\liminf_{t}R_{t}$). Observe that the trajectory of $\nbracket{R_t}_{t\geq 1}$ is governed by the mixture weights $\nbracket{\alpha_t}_{t\geq 1}$ and sample sizes $\nbracket{n_t}_{t\geq 0}$. Under different regimes of $\nbracket{\alpha_t}_{t\geq 1}$ and $\nbracket{n_t}_{t\geq 0}$, we analyze the long-term behavior of the model risk, focusing on the following scenarios based on the degree of collapse quantified by the excess limiting risk $R_{\infty}-R_{0}$: 
\begin{enumerate}
    \item \emph{Consistency}: models consistently improves over time, achieving vanishing risk in the limit, $R_{\infty}=0$.
    \item \emph{Iterative Improvement}: models avoid collapse and continue to improve or stabilize at a risk level better than or close to the initial risk: $R_{\infty}\leq R_{0}$.
    \item \emph{Collapse}: models fail to improve, with limiting risk exceeding the standard risk, e.g. $R_{\infty} > R_{0}$.
\end{enumerate}

First we analyze the case when the mixture weight of real data is bounded away from 0. 

\begin{corollary}
\label{cor:1}
Assume the  mixture weight of real data maintains $\inf_{t}\alpha_t=\alpha > 0$. Then based on the evolution of the training size $n_t$, the following holds: \begin{enumerate}
    \item  If $n_t=n$, then $R_t\in [R_0, nR_0)$. More precisely,: $$R_{0}\leq R_\infty \leq R_{0}\left(1 - \frac{n-1}{n}(1-\alpha)^2\right)^{-1}.$$ 
    \item  If $n_t\uparrow \infty$, then  $R_\infty= 0$ (consistency for any increasing sample size sequence $n_t$).
\end{enumerate}
\end{corollary}

Note with fixed sample size, future models can never be better than the initial model, but the gap $R_{\infty}-R_{0}$ depends on the mixture weights and is smaller for bigger $\alpha$. 

\begin{remark}\label{remark: fixed alpha and n}
    The proof of the above corollary demonstrates that if $\alpha_t=\alpha$ is fixed and $n_t=n$ for all $t$, then the limiting risk is \textit{exactly} the upper bound, i.e., $R_{\infty} = R_0\left(1 - \frac{n-1}{n}(1-\alpha)^2\right)^{-1}.$
    The bound appears due to the condition that $\inf_t \alpha_t=\alpha$. When $\alpha=0$ (purely synthetic data), the limit becomes $nR_0$, while when $\alpha=1$ (fresh real data), the limit is $R_0$, as expected.
\end{remark}

Now, let's consider the situation when truth mixture weights are $\at > 0$ but decays to 0, that is $\at \downarrow 0$.

\begin{proposition}{(Decaying Truth  Mixture Weights)}\label{prop:decay}
\begin{enumerate}
    \item For any sequence $\at \downarrow 0$, a `large enough' sample size: $$ n_t > \left\lceil \frac{n_{t-1}-1}{\at(2-\at)}\right \rceil+1$$ ensures $t$-th stage improvement i.e. $R_{t}<R_{t-1}$. Consequently, if the above recursive inequality is satisfied for all $\nbracket{n_t: t\geq 1}$ , then $\nbracket{R_t}_{t\geq 0}$ is a strictly decreasing sequence and $R_{\infty}<R_{0}$. 
    \item If $\at$ satisfies $\sum_{t=1}^{\infty} \at = \infty$, then any sequence of sample sizes with $\sum_{t=1}^{\infty} 1/n_t< \infty$ ensures consistency, that is $R_{\infty}=0$. 
\end{enumerate}
\end{proposition}

\begin{remark}
The condition in Part (2) interprets as \textit{sufficiently slowly decaying} $\at$ and examples include $\alpha_t\asymp 1/\log(t)$ or $\alpha_t\asymp 1/\sqrt{t}$ or $\alpha_t\asymp 1/t$ or $\alpha_t\asymp 1/t(\log t)^\gamma$ for some $\gamma\in(0,1)$ -- in all of such cases, the choice of $n_t\asymp t^2$ ensures consistency. 
\end{remark}

The above result provides sufficient (but not necessary) conditions under which improvement or consistency is achieved at the limit $t \to \infty$. Unlike the case 2 in Corollary~\ref{cor:1}, where any sequence $n_{t}\uparrow\infty$ ensures consistency, these conditions on the evolution of sample sizes depend on the sequence of decaying mixture weights. 
Observe (from part 1) that to ensure improvement $n_t$ needs to grow faster for smaller truth mixture weight $\at$, essentially suggesting adjusted scaling laws while training with synthetic data, as shown in~\cite{dohmatob2024taletailsmodelcollapse}. One potential limitation of Proposition \ref{prop:decay} is that the sample size $n_t$ depends on the potentially unknown $\alpha_t$, which is problematic in real-life scenarios -- although this is expected. However, the purpose of the result is to demonstrate that there exists sequences $n_t$ such that estimation can indeed be improved even by agnostic model fitting with contaminated data. 


\subsubsection*{Using a classifier before training}
    Recent work \citep{zhang2024regurgitative, feng2024beyondcollapses} has proposed using human or AI detector to filter out synthetic data to avoid model collapse. We provide a brief discussion of this idea under our framework.
    Assume that at stage $t$, there is a classifier $\cC_t$ that is trained to differentiate between $P^*$ (real) and $\hat{P}_t$. The data $Y^{t+1}$ is then passed through this classifier and only the ones that are classified as `real' are retained for training $\hat{P}_{t+1}$. Assume a simplified setting where label-agnostic $\cC_t$ has type 1 error rate $e_{1t}$ and type 2 error rate $e_{2t}$ (where the null corresponds to $P^*$). Recall that under our assumptions, $Y^{t+1}$ is drawn from a mixture distribution with $\alpha_{t+1}$ proportion coming from $P^*$ and $(1-\alpha_{t+1})$ from $\hat{P}_t$. If $\tilde{Y}^{t+1}$ is the filtered data obtained from $Y^{t+1}$ by taking only those classified as coming from $P^*$, then it is easy to see that $\tilde{Y}^{t+1}$ also has a mixture distribution, with new mixture weight corresponding to $P^*$ as $$\tilde{\alpha}_{t+1}=\frac{\alpha_{t+1}(1-e_{1t})}{\alpha_{t+1}(1-e_{1t}) + (1-\alpha_{t+1})e_{2t}}$$ and the remaining $1-\tilde{\alpha}_{t+1}$ from $\hat{P}_t$. Also, note that if the size of $Y^{t+1}$ is $n_t$, then the expected size of $\tilde{Y}^{t+1}$ is $$\tilde{n}_{t+1}=[\alpha_{t+1}(1-e_{1t})+(1-\alpha_{t+1})e_{2t}] n_{t+1}.$$

    For example, for a random classifier with $e_{1t}=e_{2t}=1/2$, we observe that $\tilde{\alpha}_{t+1}=\alpha_{t+1}$, while $\tilde{n}_{t+1}=n_{t+1}/2$.

    \paragraph{Oracle classifier:} In practice, training of $\cC_t$ has to be done through samples from both $P^*$ and $\hat{P}_t$. Since $\hat{P}_t$ is the model trained at the last step, one can generate infinitely many samples from this distribution. Assume that there is an oracle who also has access to infinite samples from $P^*$. In this case, the information theoretic lower bound of the test is
    $$e_{1t} + e_{2t} \geq 1-\TV(P^*, \hat{P}_t).$$
    If the oracle classifier attains this (e.g., the Neyman-Pearson LRT) -- for a fixed type 1 error rate $e_{1t}=e_1$, it has a type 2 error rate $e_{2t}=1-e_1-\TV(P^*, \hat{P}_{t})$. In this case, the training of the model in the next stage is based on $\tilde{Y}_{t+1}$ whose mixture weight for $P^*$ is given by
    \begin{align*}
        \tilde{\alpha}_{t+1} &= \frac{\alpha_{t+1}(1-e_1)}{\alpha_{t+1}(1-e_1) + (1-\alpha_{t+1})(1-e_1 - \TV(P^*,\hat{P}_t))} \\&= \frac{\alpha_{t+1}(1-e_1)}{1-e_1 - 0.5(1-\alpha_{t+1})\|\hat{\theta}_t-\theta^*\|_1} \geq \alpha_{t+1}.
    \end{align*}
    This gives us a way to analyze $R_\infty$ in such cases, using same technique from Theorem \ref{thm:1}. See Appendix \ref{app:classifier}, \ref{app:adaptive} for more details.

\subsection{Iterative Retraining On Purely Synthetic Data}
Now, we turn to the case when there is no fresh real data except $D_{0}$, i.e.  for all $t\geq 1$, $\at=0$ and the $t$-th model $\widehat{P}_{t}$ is trained on \emph{purely} synthetic data $Y^t$, generated from previous model fit $\widehat{P}_{t-1}$. For fixed sample sizes $n_t=n$, collapse cannot be avoided and the limiting risk is $nR_0$, as discussed in Remark \ref{remark: fixed alpha and n}. The following corollary shows that increasing $n_t$ might be beneficial in this case. 

\begin{corollary}
\label{cor:2}
For iterative retraining with purely synthetic data with sample sizes $\nbracket{n_t: t\geq 0}$, the following holds: $n_{0}R_{0}\sum_{s\leq t} (1/2^s n_s)<R_t < n_{0}R_{0}\sum_{s\leq t} (1/n_s)$. 
\end{corollary}

 Unlike previous situations, there is no hope of improvement in the sense $R_t<R_0$ by growing sample size $n_t$. However, $R_t$ can be made arbitrarily close to $R_0$ (this is intuitively obvious since if we have access to unlimited samples from $P_{\theta_t}$, then $\theta_t$ can be estimated without any statistical error, hence we only incur error at the first round with finite $n_0$ samples from $P^*$).

\subsubsection{Data Accumulation}
\label{sec:accu}
To mitigate model collapse, a line of work~\cite{accumulation,kazdan2024accumulating} suggests training with accumulated data from past stages. Here we study the training evolution with accumulated data in our setting. Using the same notations as before, $Y^{0}$ be the initial data and $\widehat{\theta}^{a}_{0}:=\widehat{\theta}_{0}$ be the initial estimate. For training time $t\geq 1$, the new stage synthetic data $Y^{t}$ of size $n_{t}$, generated from the previous model fit. Then the next stage model is fitted by estimating $\widehat{\theta}^{a}_{t}$ from data accumulated through all previous stages $Y^{0}, \cdots, Y^{t}$. We derive the entire trajectory of risks $R^a_t :=\mathbb{E}\Vert\hat{\theta}^a_t - \theta^*\Vert_2^2$ under accumulated training (Theorem~\ref{thm:acccumulation}, Appendix~\ref{app:1}) and state some key findings below.

\begin{proposition}
\label{porp:accum1}
For fixed number of samples per iteration $n_t=n$, the sequence of $\nbracket{R^a_t }_{t\geq 0}$ is eventually strictly increasing and it's limit, say \(  R_\infty^{a}\) can be bounded both sides as 
\[
  R_0 \cdot \exp\left( -\frac{1}{n} \left( \frac{\pi^2}{6} - 1 \right) \right)   \le R_\infty^{a} \le R_0 \cdot \nbracket{1 + \frac{\pi^2}{6}}.
\]
\end{proposition}

\begin{proposition}
\label{porp:accum2}
For sample sizes growing at each iteration with $n_t\uparrow \infty$, the limit \( \lim_{t \to \infty} R^{a}_t = R_\infty^{a} \) exists and \( R_\infty^{a} > R_0 \). Moreover, if \( R_\infty^{a} < n_0 R_0 \), then \( R^{a}_t \) is eventually monotonically increasing.
\end{proposition}

Our findings aligns with the results in the previous works on accumulation \cite{accumulation,kazdan2024accumulating}. Proposition~\ref{porp:accum1} for fixed sample sizes $n_t=n$, gives a stronger upper-bound on the limiting risk than what we would obtain by standard training. This demonstrates the benefit of accumulation. However, also note (from Proposition~\ref{porp:accum2}) for sample sizes growing at each iteration, there is no extra benefit from accumulation and $R_t$ is eventually increasing. Added with the increasing memory requirement for the accumulation process, it seems that standard retraining should be preferred.


\section{Importance of fresh real data to improve iterative training}
\label{sec:one step}
Results from previous sections show (i) with fresh real data at each stage, improvement is always possible and consistency also holds in many cases, (ii) with no fresh real data after $D_0$, it is not possible (even with accumulation) to do better than the initial fit. In this section, we argue that this is a more general phenomenon. Consider a \textit{single step} of the general iterative training process described in Section~\ref{sec:framework}, with model class $\cP=\{P_{\theta}:\theta\in\Theta\}$. Let $\hat{\theta}_0$ be an estimator trained using data $D_0$ and $\hat{\theta}_1$ be the estimator based on $D_1$ in the following iteration. Our aim is to understand, under what conditions can $\hat{\theta}_1$ be statistically better than $\hat{\theta}_0$. The following two results respectively deal with the necessity and sufficiency of fresh samples from the true data distribution to improve estimation. Denote $\cR(\theta, P):=\bbE[\ell(\theta;Z)]$ where $\ell$ is some loss function, and the expectation is taken with respect to $Z\sim P$ and any randomness in $\theta$ (if estimator). Typical examples with $Z=(X,Y)$ include the regression setting with $\ell(\theta;Z)=(Y-X^\top \theta)^2$ (linear regression), $\ell(\theta,Z)=(Y-X^\top \theta)^2+l\norm{\theta}_1$ (LASSO regression), classification setting with $\ell(\theta,Z)=-Y\log(\sigma(X^\top\theta))-(1-Y)\log(1-\sigma(X^\top\theta))$ (for logistic regression), $\ell(\theta,Z)=\max\{0, 1-YX^\top\theta\}$ (for linear SVM) and with $Z=X$, the density estimation setting with $\ell(\theta,X)=-\log p_{\theta}(X)$. 



\begin{theorem}\label{thm:one_step1}
    If $D_1|D_0$ is independent of $P^*$ and known, then for any convex loss $\ell(\theta;X)$ (convex in $\theta$), for any estimator $\hat{\theta}_1$ based on $D_1$, there exists $\hat{\theta}_0$ based on $D_0$, such that $\cR(\hat{\theta}_0,P^*) \leq \cR(\hat{\theta}_1,P^*)$.
\end{theorem}
The result shows that if $D_1$ has no additional information about $P^*$ other than that already in $D_0$, then for any estimator based on $D_1$, one can construct an estimator based only on $D_0$ which outperforms the former. It is important that the distribution of $D_1|D_0$ is known, otherwise although $\hat{\theta}_0=\bbE[\hat{\theta}_1|D_0]$ has better risk, it is not a computable statistic from the data. As an example, suppose given $D_0$, an adversary chooses a distribution for $D_1$ without any knowledge of $P^*$ -- in this case, the learner without access to this extra information cannot compute $\bbE[\hat{\theta}_1|D_0]$. Secondly, 
the result does not say that given an estimation method (e.g. empirical risk minimization), using it on $D_0$ is as good as using it on $D_1$. In the mixture setting $P_1=\alpha P^* + (1-\alpha) \hat{P}_0$, where $\hat{P}_0$ is a function of $D_0$; hence for $\alpha=0$, $D_1|D_0$ is independent of $P^*$. It is interesting to note that the result applies even to the case of accumulation, since $D_1|D_0$ is still independent of $P^*$ (reusing `real' data from $D_0$ is very different from having `fresh real' data).


We next present a partial converse to the above result, demonstrating that availability of fresh information injected into the training actually helps the statistical estimation. In this, we consider the mixture model setting discussed in the previous section. In this case, since we only consider a single-step of the overall iterative evolution, consider an arbitrary $\hat{\theta}_0$. In the next step, under the mixture model, the underlying population distribution is $Q_1 = \alpha P^* + (1-\alpha) P_{\hat{\theta}_0}$ -- see Figure \ref{fig:likelihood} for a visualization. Consider the negative log likelihood loss $\ell(\theta,X)=-\log p_{\theta}(X)$. Minimizing the risk wrt this loss is equivalent to MLE, which at the population level is equivalent to the KL-projection of $Q_1$ on the model space $\cP$. The next result shows that at the population level, the risk can be improved if $\alpha>0$.

\begin{theorem}\label{thm:one_step2}
    For $\ell(\theta,X)=-\log p_{\theta}(X)$, given an arbitrary $\theta_0\in\Theta$, $Q_1=\alpha P^* + (1-\alpha)P_{\theta_0}$ and $\theta_1=\argmin_{\theta} \cR(\theta, Q_1)$, we have $\cR(\theta_1,P^*) \leq \cR(\theta_0, P^*)$ if $\alpha>0$.
\end{theorem}

It is worth emphasizing that the result is at a population level -- for sample-based estimators, the statistical error due to finite sample size also enters the picture. However, it shows that with sufficiently high sample size $n$, the estimator has better risk. Lastly, If $\cP$ is convex and $\cR(\theta,P^*)$ is strictly convex in a neighborhood $\cU\subset\Theta$ of $\theta^*:=\argmin_{\theta} \cR(\theta, P^*)$ and $\theta_0\in \cU$, then the inequality is strict (see Appendix \ref{app:one-step} for details).


\section{Numerical experiments}
\label{sec:simulation}
We complement our analysis with several numerical experiments. The purpose of the experiments is to understand the iterative evolution of the estimation performance $R_t$ for different types of models under different rates for the proportion of ground truth $\alpha_t$, and sample size $n_t$. We present the results for the multinomial model, which supports our results and also shed light on related heuristics. Due to space constraints, we postpone additional results for this model, along with results for the Gaussian model, Gaussian mixture model and logistic regression model to the Appendix~\ref{app: additional experiments}. The results hint at the wider applicability of the results presented in this work. Finally, we present a numerical experiment using the nano-GPT model, simulating a real-life evolution of training under synthetic data.

\begin{figure}
    \centering
    \includegraphics[width=\linewidth]{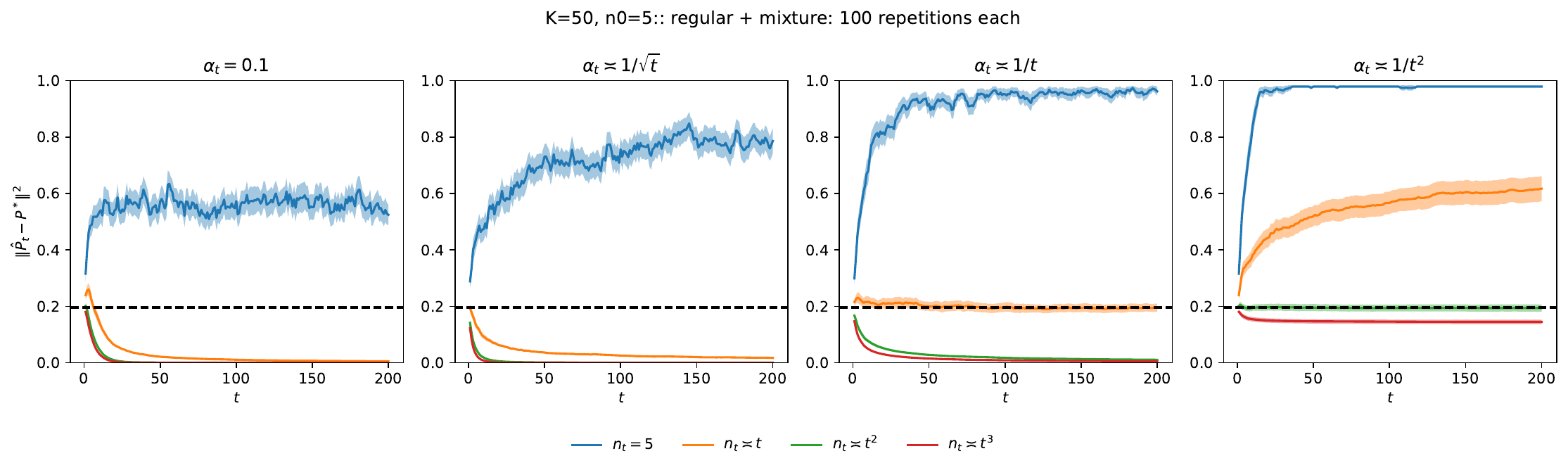}
    \includegraphics[width=\linewidth]{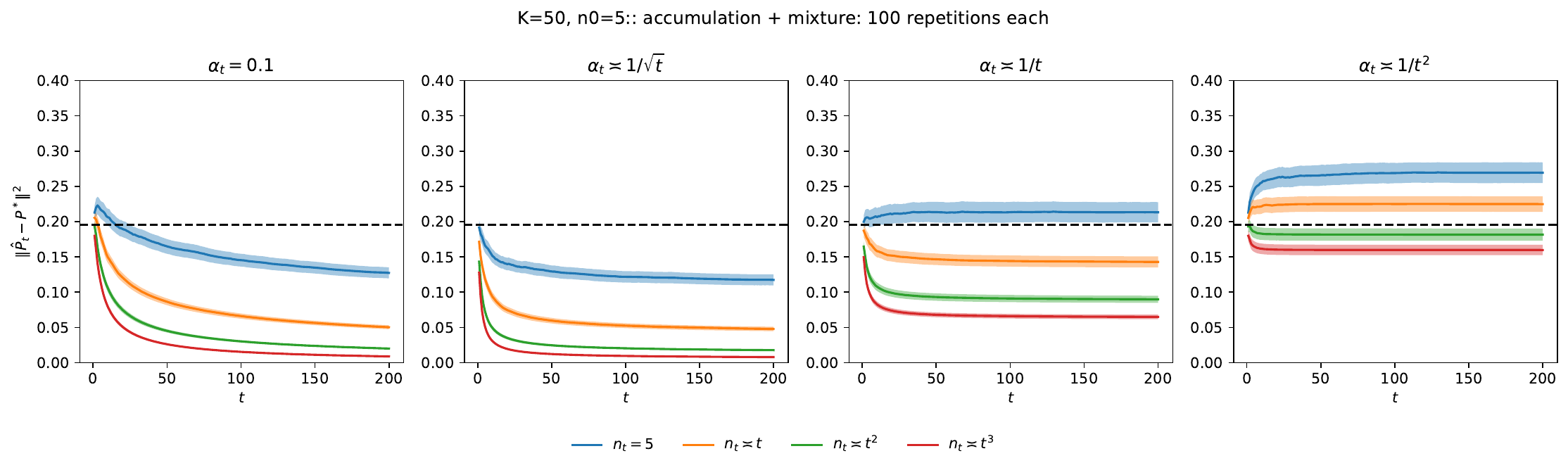}
    \caption{Iterative evolution of estimation quality for multinomial model under data arising from a mixture of a fixed ground-truth distribution and synthetic data generated from the trained model at previous iteration under various settings for $\alpha_t$ and $n_t$ -- top row shows the results without accumulation, bottom row shows results with accumulation, The black dashed horizontal line is the value of $R_0$}
    \label{fig:multinomial experiments}
\end{figure}
\subsection{Multinomial model}

We study the case of iterative training of multinomial models when the training data at each iteration is corrupted by synthetic samples generated from the trained model at the previous iteration. We choose $K=50$ and perform experiments under different combinations of the proportion of ground-truth $\alpha_t$ and sample sizes $n_t$. In particular, we consider four types of $\alpha_t$ -- (i) fixed $\alpha_t=0.1$, (ii) $\alpha_t\asymp 1/\sqrt{t}$, (ii) $\alpha_t\asymp 1/t$ and (iv) $\alpha_t\asymp 1/t^2$. For each of these, we consider four cases for $n_t$ -- (i) fixed $n_t=5$, (ii) linearly growing $n_t\asymp t$, (iii) quadratically growing $n_t\asymp t^2$ and (iv) $n_t\asymp t^3$. From experience, it is hard to include faster decay rates of $\alpha_t$ or growth-rates for $n_t$, because of numerical issues. For each of these cases, we iteratively train the model and compute the distance of the estimate of the parameter $\hat{p}\in\Delta^{K-1}$ to the fixed ground-truth parameter $p^*\in\Delta^{K-1}$, for $t=1,\dots,T=200$. Each experiment is repeated 100 times and the mean is shown, along with a confidence band for the mean. For all experiments, the data at iteration $t$ was drawn from the mixture model $\alpha_t p^* + (1-\alpha_t)\hat{p}_{t-1}$ and the MLE estimate was used (without any additional mixture-based modification). Figure \ref{fig:multinomial experiments} shows the results of our experiments -- plotting the evolution of $R_t$ over iterations $t$, based on 100 repetitions of the experiment for each setting (the mean is shown as the line and the band shows the confidence interval for the mean). The top row shows experiments without any accumulation of training data, and the bottom row shows results with accumulation of training data, that is, at iteration $t$, the MLE was computed based on \textit{all} data from iterations $s=1,2,\dots,t$. Figure \ref{fig:multinomial experiments} firstly shows that with $\alpha_t$ fixed or decaying slowly, it is possible to recover the true parameter, by training the model on corrupted data in an agnostic fashion. Secondly, for rapidly decaying $\alpha_t$, although we do not see consistency, for fast $n_t\asymp t^3$, we see that $R_{\infty}<R_0$, demonstrating that iterative training actually improves estimation. Finally, a direct comparison between training with and without accumulation shows that it improves the performance when $n_t$ is fixed, but for increasing $n_t$ this is not exactly true -- for example, for $\alpha_t\asymp 1/t$, neither choice of $n_t\asymp t^2$ or $t^3$ leads to consistency.

\begin{figure}
    \centering
    \includegraphics[width=0.5\linewidth]{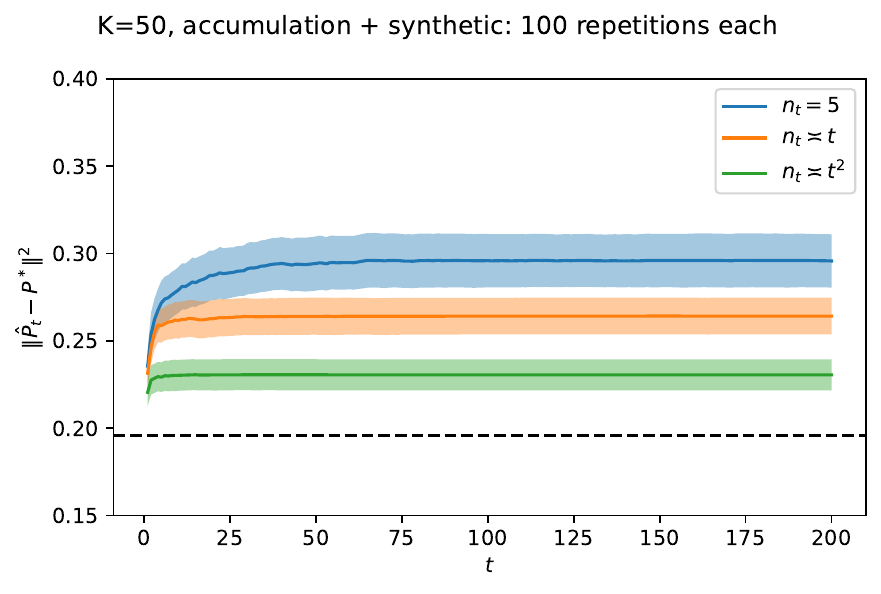}
    \caption{Iterative evolution of estimation quality for multinomial model under purely synthetic regime, when models are trained on accumulated data.}
    \label{fig:simulation accumulation}
\end{figure}

We also perform a simulation study to understand the effect of accumulation under training with purely synthetic data regime (corresponding to $\alpha_t=0$). The result is shown in Figure \ref{fig:simulation accumulation}, for different choices of $n_t$. Note that $n_t$ is the number of synthetic samples drawn from the trained model at the previous iteration -- with accumulation, the total amount of data on which the model at iteration $t$ is trained on is $\sum_{s\leq t} n_s$. The results match our findings in Proposition \ref{porp:accum2}, we see that $R_t$ increases with $t$, till it stabilizes. A faster rate of growth for $n_t$ simply reduces this limit, but the performance is always worse compared to $R_0$. This can be seen as a specific case of our general results in Section \ref{sec:one step}, illustrating that unless fresh information from the ground-truth distribution is available, the estimation quality will only worsen over iterations.

\subsection{GPT-2 learning on real text data}
In this section, we explore a more complex transformer-based language model and demonstrate the iterative evolution of it, being trained under different settings. For these experiments, we use the Wikipedia movie plots dataset from Kaggle\footnote{\texttt{www.kaggle.com/datasets/jrobischon/\\wikipedia-movie-plots}}. It consists of movie plots (we call these documents) of around 34000 movies. For each experiment, we hold out 500 documents and all training is done based on observations from the remaining documents and synthetically generated documents (exact details given in the settings below). For each setting, we do recursive training for $T$ rounds, evaluating the fitted model based on perplexity on a held-out data corpus. 

 We employ a moderate-sized GPT-2 language model using the nanopgt implementation by Andrej Karpathy \citep{Karpathy2022} (additional details provided in Appendix \ref{app: gpt2 details}).  The following four settings are considered (with $n_0=1000$):
\begin{enumerate}
    \item  \textit{Setting 0 (Only synthetic samples):} Use $n_0$ documents from the dataset (fresh data) to train the model. Subsequently, generate $n_t= n_0 t$ synthetic documents (each with around 150 words) from the last fitted model and use that to train the next model.
    \item \textit{Setting 1 (Accumulation):} Use $n_0$ documents from the dataset to train the model. Subsequently, generate $n_0$ synthetic documents from the last-fitted model and accumulate with all previous training data.
    \item \textit{Setting 2 (Fixed $\alpha$)} Use $n_0$ documents from the dataset to train the model. Subsequently, generate $n_t/4$ synthetic documents from last-fitted model and add $3n_t/4$ fresh documents from the real data corpus (unused previously), to get total sample size $n_t=1000+200t$ at time $t$. Thus, at each iteration, we have a fixed $\alpha_t=0.25$.
    \item \textit{Setting 3 (Decaying $\alpha$):} Use $n_0$ documents from dataset (fresh data) to train the model. Subsequently, at $t=1,2,\dots$ generate $200 t$ synthetic documents from last-fitted model and add $n_0$ \textit{fresh} documents from the real data corpus (unused previously). Thus, at iteration $t$, we have proportion of fresh data $\alpha_t=1000/n_t$, which is decreasing since $n_t=1000+200t$.
\end{enumerate}

\begin{figure}
    \centering
    \includegraphics[width=0.48\linewidth]{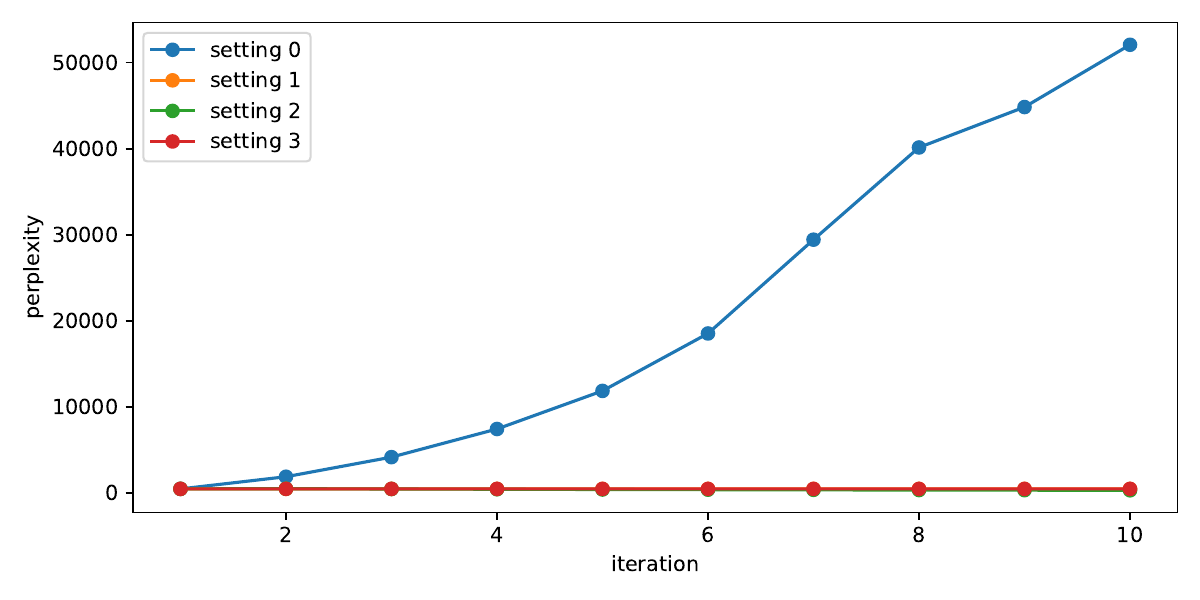}
    \includegraphics[width=0.48\linewidth]{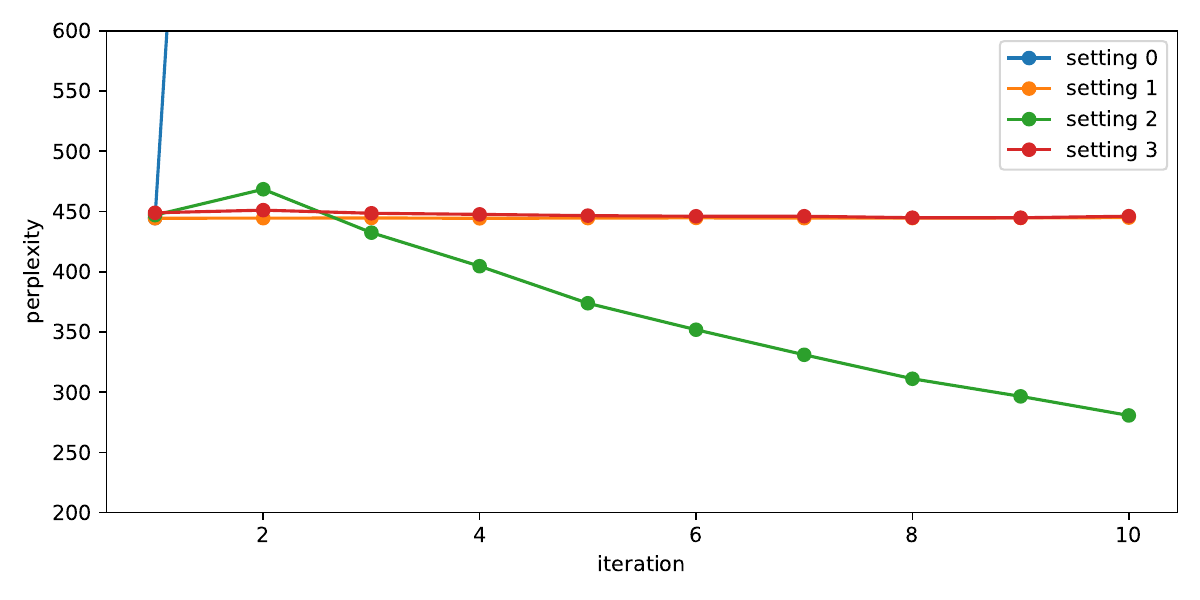}
    \caption{Iterative evolution of GPT2-like language model evaluated based on perplexity on held-out data corpus in 4 different settings. The two plots are of the same 4 settings (blue: setting 0, orange: setting 1, green: setting 2 and red: setting3), just different scales in the y-axis.}
    \label{fig:simulation_gpt}
\end{figure}

The findings of our simulation studies are given in Figure \ref{fig:simulation_gpt} -- each colored plot indicates a particular setting, and we plot the average perplexity over the 10 random repetitions of the experiment. Similar to our findings from the previous section (and the results in Appendix~\ref{app: additional experiments}), setting 0 leads to catastrophic collapse, as can be seen by the increase in the perplexity. The other settings can be seen more clearly in the right plot in Figure \ref{fig:simulation_gpt}. For both settings 1 and 3, we see that the model iteratively does not improve; however, it does not deteriorate too much. Lastly, we see that in setting 3, the model actually improves over iterations, shown by the decrease in perplexity. 

\section{Conclusion \& Future Directions}
\label{sec:dis}

We investigate model collapse in generative models and identify conditions under which statistical estimation can improve during iterative training with mixed real and synthetic data. In a simplified next-token prediction setting, we precisely characterize error evolution and show that consistency and improved estimation are achievable under suitable data mixing. Simulations suggest these findings generalize beyond the toy setting. Crucially, we show that without continual infusion of real data, performance cannot improve, making retraining ineffective. Our analysis highlights how the balance between real data fraction and sample size governs long-term performance, offering testable predictions for practical applications.

While sample sizes can often be tracked, estimating the mixture weights remains a challenge. Leveraging recent advances in synthetic data detection, e.g., \cite{radvand2025zeroshotstatisticaltestsllmgenerated}), a crucial future direction lies in estimating mixture proportions. Using which one can obtain corrected estimates $\widehat{\theta}_{t+1}^{c}=\et+ {\watt}^{-1}(\ett-\et)$, potentially improving estimation. Although we do not directly control the proportion of synthetic data at each iteration, it is reasonable to expect lower synthetic contamination when the model is less accurate. In practice, e.g., training language models on web-scale corpora--implicit human or algorithmic filtering naturally suppressing poor generations. This implies that the mixture weight \(1 - \alpha_t\), reflecting the prior model’s contribution, decreases with increasing risk \(R_{t-1}\). Such \emph{naturally adaptive control} reduces the influence of weak models, preventing collapse and keeping the limiting risk \((R_\infty)\) near or below the initial risk \(R_0\), without requiring assumptions on \(n_t\) (see Appendix~\ref{app:adaptive}).

While our framework models fresh information via a mixture distribution, there are other ways to get new information—most notably, reinforcement learning from human feedback (RLHF) to fine-tune LLMs or curate synthetic data~\cite{feng2024beyondcollapses, ferbach2024selfconsuminggenerativemodelscurated}. Our iterative setup also connects to information passing scenarios involving limited fresh input, e.g., communicating LLMs or echo chambers \cite{cinelli2021echo}. In  future, it would be interesting to extend our theoretical framework to study a broader range of interactive training evolutions. Finally, assuming a fixed true data distribution may be unrealistic in evolving language environments. As AI-generated content fills up online text sources, the output from LLMs might also influence the language model of humans, i.e., the true data distribution. This feedback loop suggests a co-evolution of human and machine language, raising important questions about long-term model dynamics.


\section*{Impact Statement}
\textit{``This paper presents work whose goal is to advance the field of Machine Learning and AI. There are many potential societal consequences of our work, none which we feel must be
specifically highlighted here.''}

\bibliography{ref}
\bibliographystyle{icml2026}

\newpage
\appendix
\onecolumn

\section{Proofs \& Additional Results in Section \ref{sec:language model}}
\label{app:1}

\begin{proof}[Proof of Theorem~\ref{thm:1}]
For $t\geq0$, firstly $Y^{t} = \nbracket{Y^{t}_{1}, \ldots, Y^{t}_{n_{t}}}$ is $\iid$ data from $\text{Cat}(\theta_{t})$ is used to make estimates $$\widehat{\theta}_{t}(k) = \frac{\sum_{j=1}^{n_{t}}\mathds{1}[Y^{t}_{j}=k]}{n_{t}}.$$ In first stage, $n_{0}\widehat{\theta}_{0}(k) \sim \text{Binom}(n_{0},\theta^{*}(k))$ thus $\bbE{\widehat{\theta}_{0}(k)} = \theta^{*}(k)$ and $\Var \widehat{\theta}_{0}(k) =  \frac{\theta^{*}(k) (1-\theta^{*}(k))}{n_{0}}$. Now for $t\geq 0$, $$\widehat{\theta}_{t+1}(k) = \frac{\sum_{j=1}^{n_{t+1}}\mathds{1}[Y^{t+1}_{j}=k]}{n_{t+1}},$$ and  $\mathds{1}[Y^{t+1}_{j}=k] \mid \widehat{\theta}_{t} \sim \text{Ber}(\att \theta^{*}(k) + (1-\att) \widehat{\theta}_{t}(k)$), thus $$\widehat{\theta}_{t+1}(k)\mid \widehat{\theta}_{t}  \sim \frac{\text{Binom}(n_{t+1},\att \theta^{*}(k) + (1-\att) \widehat{\theta}_{t}(k) )}{n_{t+1}}.$$ Then \begin{align*}
    &\bbE{\theta}_{t+1}\mid \widehat{\theta}_{t} = \widehat{\theta}_{t} \\
    \implies & \bbE{\theta}_{t+1} = \bbE{\theta}_{t} \cdots =  \bbE{\theta}_{0}\\
    \implies & \bbE{\theta}_{t} = \theta^{*}.
\end{align*} 
Now using the total law of variance, we note that the variance of estimate $\widehat{\theta}_{t+1}(k)$ satisfys the following recursion,
\begin{align*}
    &\Var \widehat{\theta}_{t+1}(k)= \bbE{ \Var \widehat{\theta}_{t+1}(k)\mid \widehat{\theta}_{t} } + \Var \bbE{\widehat{\theta}_{t+1}(k)\mid \widehat{\theta}_{t} } \\
    &=\frac{1}{n_{t+1}}\nbracket{\theta^{*}(k)(1- \theta^{*}(k)) - (1-\att)^{2} \Var \widehat{\theta}_{t}(k)}\\
    & + (1-\att)^{2} \Var \widehat{\theta}_{t}(k)\\
    &= \frac{\theta^{*}(k)(1- \theta^{*}(k))}{n_{t+1}}+ \frac{n_{t+1}-1}{n_{t+1}}(1-\att)^{2}\Var \widehat{\theta}_{t}(k).
\end{align*}Moreover, \begin{align*}
    R_t &= \mathbb{E}\Vert\hat{\theta}_t - \theta^*\Vert_2^2 \\
    &= \sum_{k=1}^{K}\bbE\rbracket{\nbracket{\hat{\theta}_{t}(k) - \theta^{*}(k)}^2}\\
    &= \sum_{k=1}^{K}\Var \widehat{\theta}_{t}(k).
\end{align*}Hence summing the recursion for $\Var \widehat{\theta}_{t+1}(k)$ over $k\in [K]$, we get the desired recursion for $R_t$.  
\end{proof}


\begin{proof}[Proof of Corollary~\ref{cor:1}]
For Case 1 with fixed number of samples $n_{t} = n$ and $\at=\alpha$, the recursion from Theorem~\ref{thm:1} simply becomes $$R_{t} = a+bR_{t-1},$$  with $$a=R_{0} \text{ and } b=\frac{n-1}{n}(1-\alpha)^{2}<1.$$  Now solving the linear recursion \begin{align*}
R_{t}-l &= b(R_{t-1}-l) \\
l - b\lambda &= a \implies l = \frac{a}{1-b}\\
R_{t} &= l + b^{t}\nbracket{R_{0}-l} \\ \\
\implies R_{t}&\to l \text{ as } t\to \infty \\
\text{ with } l = &  \ R_{0}\left(1 - \frac{n-1}{n}(1-\alpha)^{2}\right)^{-1}. 
\end{align*} Now notice if $\alpha_t > \alpha >0$ for all $t\geq 1$, $R_{t}$ would decay faster than what it would with fixed $\alpha_t=\alpha$. Hence the above limit for $R_t$ with constant $\alpha_t = \alpha$ would be a limiting upper-bound  hold for any sequence of $\nbracket{\alpha_{t}}_{t\geq 1}$ if $\alpha_t > \alpha$.

For Case 2 with varying sample size $(n_t)_{t\geq1}$, note that
$$R_t \leq \frac{c_0}{n_t} + \left(1- \frac{1}{n_t}\right)\epsilon R_{t-1}$$
where $c_0=n_0R_0$ and $\epsilon=(1-\alpha)^2 < 1$. The above gives
$$R_t < \frac{c_0}{n_t} + \epsilon R_{t-1}$$
which recursively gives
\begin{align*}
    R_t &< \frac{c_0}{n_t} + \epsilon \frac{c_0}{n_{t-1}} + \epsilon^2 \frac{c_0}{n_{t-2}} + \dots + \epsilon^{t-1}\frac{c_0}{n_1} + \epsilon^t R_0 \\
    &= c_0 \sum_{s=1}^t \frac{\epsilon^{t-s}}{n_s} + \epsilon^t R_0.
\end{align*}
Thus, if $n_t\uparrow \infty$, then $R_t\to 0$ as $t\to\infty$.

Details: Consider the series $S_t=\sum_{s=1}^{t} \frac{\epsilon^{t-s}}{n_s}$. Fix $\delta>0$. Show that for large $t$, $S_t<\delta$. Towards this, note that $S_t = \sum_{s=0}^{T-1} \frac{\epsilon^{s}}{n_{t-s}} + \sum_{s=T}^{t-1} \frac{\epsilon^{s}}{n_{t-s}}$. The first part is upper bounded by $T/n_{t-T}$. The second part is upper bounded by $\epsilon^T/(1-\epsilon)n_1$, giving an overall upper bound of $S_t<T/n_{t-T} + \epsilon^T/(1-\epsilon)n_1$ for any choice of $T<t$. Choose $T(t)$ as a function of $t$, such that $T(t)\uparrow \infty$ slower than $n_t$. Then, for large enough $t$, we get $S_t<\delta$.

\begin{align*}
    R_{t} > \frac{n_{0}}{n_{t}}R_{0} + \frac{1}{2}R_{t-1} 
\end{align*} so for large enough $t$ as $n_t\uparrow \infty$ we can find a constant $\beta < 1$ such that \begin{align*}
    R_{t} &< \beta R_{t-1} \implies  R_{t} = \beta^{t} R_{0} \to 0 \text{ as } t \to \infty.\\
\text{Also, }  & \frac{n_{0}}{n_{t}}R_{0} + \frac{n_{t}-1}{n_{t}}\epsilon R_{t-1} < \beta R_{t-1}\\
    \frac{n_{0}}{n_{t}}R_{0} &< \rbracket{\beta n_{t} - \epsilon (n_{t}-1)}R_{t-1} \\
    &= \rbracket{\epsilon + n_{t}(\beta -\epsilon)}R_{t-1} \\
\end{align*} where $\epsilon = (1-\alpha)^{2}$ and $\beta = \frac{1+\epsilon}{2} \in (\epsilon,1)$.
\end{proof}

\begin{proof}[Proof of Corollary~ \ref{cor:2}]
If we have $\alpha_t=0$ the recursion from Theorem~\ref{thm:1} simplifies and then further telescoping gives
\begin{align*}
n_{t}(R_{t}-R_{t-1}) &=  n_{0}R_{0} - R_{t-1}\\
\implies  n_{0}R_{0}\sum_{s\leq t} (1/2^s n_s)&<R_t < n_{0}R_{0}\sum_{s\leq t} (1/n_s).
\end{align*} 
\end{proof}

\begin{proof}[Proof of Proposition~\ref{prop:decay}]
\emph{Part 1: } 
Define $\xi_t >0$ such that  $$\xi_t = 1-(1-\alpha_t)^2, \ \forall t \geq 1$$Note that $\xi_t=2\alpha_t - \alpha_t^2$ decays at the same rate as $\alpha_t$. Let's begin with the general recurrence relation of $R_t$. To simplify notations replace $(1-\alpha_t)^2$ by $1-\xi_t$ and then divide both sides by $R_{t-1}$.
 \begin{align*}
      R_t &= \frac{n_{0}R_{0}}{n_t} + \frac{n_t-1}{n_t} (1-\alpha_t)^2 R_{t-1} \\
     \implies  R_t  &= \frac{n_{0}R_{0}}{n_t} + \frac{(n_t-1)(1-\xi_{t})}{n_t}R_{t-1} \\
     \implies  \frac{R_t}{R_{t-1}}   &= \frac{n_{0}R_{0}}{n_t}\frac{1}{R_{t-1}}+ \frac{(n_t-1)(1-\xi_{t})}{n_t} \\  
    \implies  \frac{R_t}{R_{t-1}}   &\leq \frac{n_{0}R_{0}}{n_t}\frac{n_{t-1}}{n_{0}R_{0}}+ \frac{(n_t-1)(1-\xi_{t})}{n_t} \\
    &= \frac{n_{t-1}+n_{t}-\xi_{t}(n_t-1)-1}{n_t} =: \eta_{t}.
\end{align*} The inequality used above is because $\forall t\geq0$,  $$R_t = \frac{n_{0}R_{0}}{n_t} + \frac{n_t-1}{n_t} (1-\alpha_t)^2 R_{t-1} \implies R_{t}\geq \frac{n_{0}R_{0}}{n_{t}}.$$
 Observe $\eta_{t}>0$ and \begin{align*}
    \eta_{t}<1 &\iff n_{t-1}+n_{t}-\xi_{t}(n_t-1)-1 < n_t \\
    &\iff n_t > \frac{n_{t-1}-1}{\xi_t}+1,
\end{align*} given $\xi_{t}$ we recursively construct $n_t$ to satisfy the above inequality. Now, telescoping products directly gives \begin{align*}
  \frac{R_t}{R_{t-1}} &\leq \frac{n_{t-1}+n_{t}-\xi_{t}(n_t-1)-1}{n_t} =: \eta_{t} \\
 \implies R_{t} \leq& \eta_{t}R_{t-1} 
   \implies R_{t} < R_{t-1}.
\end{align*}Thus growing the number of samples $n_t$ satisfies the above recursive bound for all $t\geq 1$, then $\nbracket{R_t}_{t\geq 0}$ is a decreasing sequence. Now as $R_t$ is also lower-bounded by 0, we can conclude $$\lim_{t\to\infty}R_t<R_{0},$$ due to the monotone convergence theorem.

\emph{Part 2:}     Let us denote
    \begin{align*}
        a_t &:= \frac{n_0}{n_t} \to 0 \\
        b_t &:= \frac{n_{t}-1}{n_t}(1-\alpha_t)^2 \in (0,1).
    \end{align*}
    The recurrence takes the standard non-homogeneous linear form
    $$R_t = a_t R_0 + b_t R_{t-1},$$
    which can be unfolded to give
    $$R_t = R_0 \underbrace{\sum_{k=0}^t a_{t-k}\prod_{j=t-k+1}^t b_j}_{s_t} + R_0 \underbrace{\prod_{j=1}^t b_j}_{p_t}.$$
    We focus on the two terms individually and show that they go to 0.

    Let us focus on the second product term $p_t$. Note that
    \begin{align*}
    p_t &= \prod_{j=1}^t b_j = \left(\prod_{j=1}^t \frac{n_{j}-1}{n_j}\right)\left(\prod_{j=1}^t (1-\alpha_j)^2\right) \\&= \prod_{j=1}^t \left(1 - \frac{1}{n_j}\right)\times \prod_{j=1}^t (1-\alpha_j)^2.
    \end{align*}
    Both the products are bounded by 1 trivially. The first product goes to 0 if $\sum_t 1/n_t = \infty$, since $\log (1-1/n_t) \leq -1/n_t$. The second product goes to 0 if $\sum_t \alpha_t = \infty$, since $\log (1-\alpha_j)^2 \leq -2\alpha_j$ for $\alpha_j\in (0,1)$. Thus the term $p_t$ goes to 0 if
    $$\sum_t 1/n_t = \infty \quad \text{OR} \quad \sum_t \alpha_t = \infty.$$

    Now, consider the first term, which is the sum $s_t$. Substituting $m=t-k$, we have
    \begin{align*}
        s_t &= \sum_{m=0}^t a_m \prod_{j=m+1}^t b_j \\
        &= \sum_{m=0}^t \frac{n_0}{n_m}\prod_{j=m+1}^t \left(1 - \frac{1}{n_j}\right)(1-\alpha_j)^2.
    \end{align*}
    Using $\log(1-x)\leq -x$ for $x\in(0,1)$, we obtain
    \begin{align*}
        \log (b_i) &\leq -\frac{1}{n_i} - 2\alpha_i \\
        \Rightarrow \prod_{j=m+1}^t b_j &\leq \exp\left(-\sum_{j=m+1}^t \left(\frac{1}{n_j}+2\alpha_j\right)\right)\\
        &\leq \exp\left(-2\sum_{j=m+1}^t \alpha_j\right).
    \end{align*}
    Thus, (ignoring the $n_0$ fixed term)
    $$s_t \leq \sum_{m=1}^t \frac{1}{n_m}\exp\left(-2\sum_{j=m+1}^t \alpha_j\right)=:\bar{s}_t.$$

    Now, fix $\epsilon>0$. We show that $\bar{s}_t<\epsilon$ for sufficiently large $t$. For this, we split the sum into two parts, for some large constant $J>0$ (independent of $t$) to be specified later.
    \begin{align*}
        \bar{s}_t &= \sum_{j=1}^{J} \frac{1}{n_j}\exp\left(-2\sum_{i=j+1}^t \alpha_i\right) + \\
        &\sum_{j=J+1}^t\frac{1}{n_j}\exp\left(-2\sum_{i=j+1}^t \alpha_i\right) =: I_1 + I_2.
    \end{align*}
    For $I_1$, this has fixed number of $J$ terms. For each $j\in[1,J]$, the part in the exponential $\sum_{i=j+1}^t\alpha_i\to \infty$ as $t\to\infty$ with $j$ fixed. Thus, there is large enough $T$ such that for $t\geq T$, $I_1\leq \epsilon/2$.

    Now, the constant $J$ is chosen in a way such that
    \begin{align*}
        &\forall j>J, \sum_{i=j+1}^\infty \alpha_i >A, \text{ for some fixed large } A.
    \end{align*}
    This is possible as $\sum_t \alpha_t =\infty$. The choice of $A$ is such that 
    $$\exp(-2A)C<\epsilon/2\iff A > \frac{1}{2}\log(2C/\epsilon)$$
    where $C = \sum_{t=1}^\infty 1/n_t$. Thus,
    \begin{align*}
        \sum_{j=J+1}^t \frac{1}{n_j}\exp\left(-2\sum_{i=j+1}^t \alpha_i\right) &\leq \exp(-2A) \sum_{j=J+1}^\infty \frac{1}{n_j}\\
        &<\epsilon/2
    \end{align*}
    Thus, gives $\bar{s}_t\leq \epsilon$. Since $\epsilon$ was arbitrary, we conclude $s_t\leq \bar{s}_t \to 0$ as $t\to\infty$. Thus, $R_t = p_t+s_t\to 0$.
\end{proof}

\begin{theorem}
\label{thm:acccumulation}
In the above setting of iterative training with accumulation, $R^{a}_t:=\mathbb{E}\Vert\hat{\theta}^{a}_t - \theta^*\Vert_2^2$ measuring the estimation quality satisfies the following recurrence 
\begin{align}
\label{rec: accu}
R^{a}_{t+1} = \frac{\nbracket{n_{0}+\cdots+n_{t+1}}^{2}-n_{t+1}}{\nbracket{n_{0}+\cdots+n_{t+1}}^{2}}R^{a}_{t} \\+ \frac{n_{t+1}n_{0}R_{0}}{\nbracket{n_{0}+\cdots+n_{t+1}}^{2}-n_{t+1}}    
\end{align} where  $R^{a}_{0}=R_0=\frac{1-\Vert\theta^*\Vert_2^2}{n_{0}}$. 
\end{theorem}

\begin{proof}[Proof of Theorem~\ref{thm:acccumulation}]
Firstly the estimates from accumulated data will have the form \begin{align*}
        \ett(k) &= \frac{\sum_{i=0}^{t+1}\sum_{j=1}^{n_{i}}\mathds{1}\rbracket{Y^{i}_{j}=k}}{n_{0}+\cdots+n_{t+1}}  
    \end{align*}which can be written as 
\begin{align*}
   & \frac{\sum_{j=1}^{n_{t+1}}\mathds{1}\rbracket{Y^{t+1}_{j}=k}+\sum_{i=0}^{t}\sum_{j=1}^{n_{i}}\mathds{1}\rbracket{Y^{i}_{j}=k}}{n_{0}+\cdots+n_{t+1}}\\
    =& \frac{\sum_{j=1}^{n_{t+1}}\mathds{1}\rbracket{Y^{t+1}_{j}=k}+(n_{0}+\cdots+n_{t})\et(k)}{n_{0}+\cdots+n_{t+1}}\\
    \implies &\bbE{\widehat{\theta}_{t+1}(k)\mid \widehat{\theta}_{t}}=  \frac{n_{t+1}\et(k) + (n_{0}+\cdots+n_{t})\et(k)}{n_{0}+\cdots+n_{t+1}}\\& \quad \quad \quad \quad \ \ \ \ \ \ \ \ \ \ \ \ = \et(k).
\end{align*}
Given $\et$, $\mathds{1}\rbracket{Y^{t+1}_{j}=k}$ is $\text{Ber}(\et(k))$
In first stage, $\widehat{\theta}_{0}(k) \sim \frac{\text{Binom}(n_{0},\theta^{*}(k))}{n_{0}}$ thus $\bbE{\widehat{\theta}_{0}(k)} = \theta^{*}(k)$, $\Var \widehat{\theta}_{0}(k) =  \frac{\theta^{*} (1-\theta^{*})}{n_{0}}$. 
Then \begin{align*}
    &\bbE{\theta}_{t+1}\mid \widehat{\theta}_{t} = \widehat{\theta}_{t} \\
    \implies & \bbE{\theta}_{t+1} = \bbE{\theta}_{t} \cdots =  \bbE{\theta}_{0}\\
    \implies & \bbE{\theta}_{t} = \theta^{*}.
\end{align*}
Now using the total law of variance $\Var \widehat{\theta}_{t+1}(k)$ is
\begin{align*}
    &= \bbE{ \Var \widehat{\theta}_{t+1}(k)\mid \widehat{\theta}_{t} } + \Var \bbE{\widehat{\theta}_{t+1}(k)\mid \widehat{\theta}_{t} } \\
    &=\frac{n_{t+1}\bbE{\et(k)(1-\et(k))}}{\nbracket{n_{0}+\cdots+n_{t+1}}^{2}} + \Var \widehat{\theta}_{t}(k)\\
    &=\frac{n_{t+1}}{\nbracket{n_{0}+\cdots+n_{t+1}}^{2}} \rbracket{\theta^{*}(k)(1-\theta^{*}(k))- \Var \widehat{\theta}_{t}(k)}\\ &+ \Var \widehat{\theta}_{t}(k).
\end{align*}Hence summing both sides over $k\in [K]$, we get the desired recursion for $R_t = \sum_{k=1}^{K}\Var \widehat{\theta}_{t}(k)$.
\end{proof}

\begin{proof}[Proof of Proposition~\ref{porp:accum1}]
We proceed in three steps.

\emph{Part 1: (Limit exists \& \( R_\infty > R_0 \)) } 

From~\eqref{eq:recurrence},
\begin{align*}
|R^{a}_{t+1} - R^{a}_t| &= \left| \frac{n_{t+1} n_0 R_0}{S_{t+1}^2 - n_{t+1}} - \frac{n_{t+1}}{S_{t+1}^2} R^{a}_t \right| \\
&\le \frac{n_{t+1} n_0 R_0}{S_{t+1}^2 - n_{t+1}} + \frac{n_{t+1}}{S_{t+1}^2} R^{a}_t \\
\end{align*}
Since $n_t\uparrow\infty$, note \( S_{t+1} \to \infty \) and \( n_{t+1} = o(S_{t+1}^2) \), it follows that
\[
|R^{a}_{t+1} - R^{a}_t| \le \frac{C}{t^2} \quad \text{for large } t.
\]
Hence \( \sum_t |R^{a}_{t+1} - R^{a}_t| < \infty \), so \( R^{a}_t \) is Cauchy and converges:
\[ \lim_{t \to \infty} R^{a}_t = R_\infty. \]

Now to show \( R_\infty > R_0 \), we split into two cases:
\underline{Case 1:} \( R_\infty < n_0 R_0 \). Then for large enough $T$, for all \( t \ge T \),
\[ R^{a}_t < n_0 R_0 - \delta \quad \text{for some } \delta > 0. \]

We now approximate the recurrence:
\begin{align*}
R^{a}_{t+1} - R^{a}_t &= \frac{n_{t+1}}{S_{t+1}^2} \left( \frac{n_0 R_0}{1 - \frac{n_{t+1}}{S_{t+1}^2}} - R^{a}_t \right) \\
&= \frac{n_{t+1}}{S_{t+1}^2} (n_0 R_0 - R^{a}_t + o(1)) > \frac{n_{t+1}}{S_{t+1}^2} \cdot \frac{\delta}{2} > 0.
\end{align*}
So \( R^{a}_t \) is strictly increasing eventually. Moreover, \( R^{a}_t > R_0 \) for all \( t \ge 1 \) since the additive term in \eqref{eq:recurrence} is strictly positive. Therefore:
\[ R_\infty > R_0. \]

\underline{Case 2:} if \( R_\infty > n_0 R_0 \), then clearly \( R_\infty > R_0 \), as \( n_0 \ge 1 \). Thus, in both cases:
\[ R_\infty > R_0 \]

\emph{Part 3: (Eventual monotonicity.)} 

In Case 1, we showed that if \( R_\infty < n_0 R_0 \), then \( R^{a}_t \) is eventually strictly increasing. If instead \( R^{a}_t \ge n_0 R_0 \) eventually, then it's trivially non-decreasing. Hence, in either case, \( R^a_t \) is eventually monotone and converges to a limit \( > R_0 \).
\end{proof}

\begin{proof}[Proof of Proposition~\ref{porp:accum2}]
The recurrence in~\eqref{rec: accu} can be re-written as 
\begin{equation}\label{eq:recurrence}
R^{a}_{t+1} - R^{a}_t =\frac{n_{t+1} n_0 R_0}{S_{t+1}^2 - n_{t+1}} - \frac{n_{t+1}}{S_{t+1}^2} R^{a}_t,
\end{equation}
where \( S_{t+1} := n_0 + n_1 + \cdots + n_{t+1} \). 
Let us define the coefficients:
\[
a_t := 1 - \frac{1}{n(t+2)^2}, \quad b_t := \frac{n}{n(t+1)^2 - 1} R_0.
\]
Then \( R^{a}_{t+1} = a_t R^{a}_t + b_t \). Since \( a_t \to 1 \) and \( b_t = \mathcal{O}(1/t^2) \), the additive increments vanish sufficiently fast, and the sequence \( R^{a}_t \) is Cauchy. Hence, the limit exists.

We can write:
\[
R^{a}_t = R_0 \prod_{s=0}^{t-1} a_s + \sum_{k=0}^{t-1} b_k \prod_{s=k+1}^{t-1} a_s.
\]

\emph{(Lower-bound)}
The product satisfies:
\begin{align*}
 &\prod_{s=0}^\infty a_s = \prod_{s=0}^\infty \left(1 - \frac{1}{n(s+2)^2}\right) \\
 \ge & \exp\left( -\frac{1}{n} \sum_{s=2}^\infty \frac{1}{s^2} \right) = \exp\left( -\frac{1}{n} \left( \frac{\pi^2}{6} - 1 \right) \right).
\end{align*}

Thus:
\[
R_\infty \ge R_0 \cdot \exp\left( -\frac{1}{n} \left( \frac{\pi^2}{6} - 1 \right) \right) > 0.
\]

\emph{(Upper-bound)}
From the closed-form expression
\[
R^a_t = R_0 \prod_{s=0}^{t-1} a_s + \sum_{k=0}^{t-1} b_k \prod_{s=k+1}^{t-1} a_s,
\]
we observe that \( a_s < 1 \), so all products \( \prod_{s=\cdot}^{\cdot} a_s \le 1 \). Using \( b_k = \frac{n}{n(k+1)^2 - 1} R_0 \le \frac{R_0}{(k+1)^2 - \frac{1}{n}} \), we get
\[
R^a_\infty \le R_0 + R_0 \sum_{k=1}^{\infty} \frac{1}{k^2 - \frac{1}{n}} = R_0 \cdot U_n.\]
Since \( \sum_{k=1}^\infty \frac{1}{k^2 - \frac{1}{n}} \le \sum_{k=1}^\infty \frac{1}{k^2} = \frac{\pi^2}{6} \), we obtain the claimed upper bound.

\emph{(Monotonicity)}
Finally, we examine the difference:
\[
R^{a}_{t+1} - R^{a}_t = -\frac{1}{n(t+2)^2} R^{a}_t + \frac{n}{n(t+1)^2 - 1} R_0.
\]
This difference is positive for large \( t \), since the second term dominates the first due to \( R^{a}_t \) being bounded and both terms decaying like \( \mathcal{O}(1/t^2) \). Hence, \( R^{a}_t \) is eventually strictly increasing.
\end{proof}

\section{Proofs \& Results in Section \ref{sec:one step}}
\label{app:one-step}
\begin{figure}
    \centering
    \includegraphics[trim={10cm 8cm 10cm 10cm}, clip, width=0.9\linewidth]{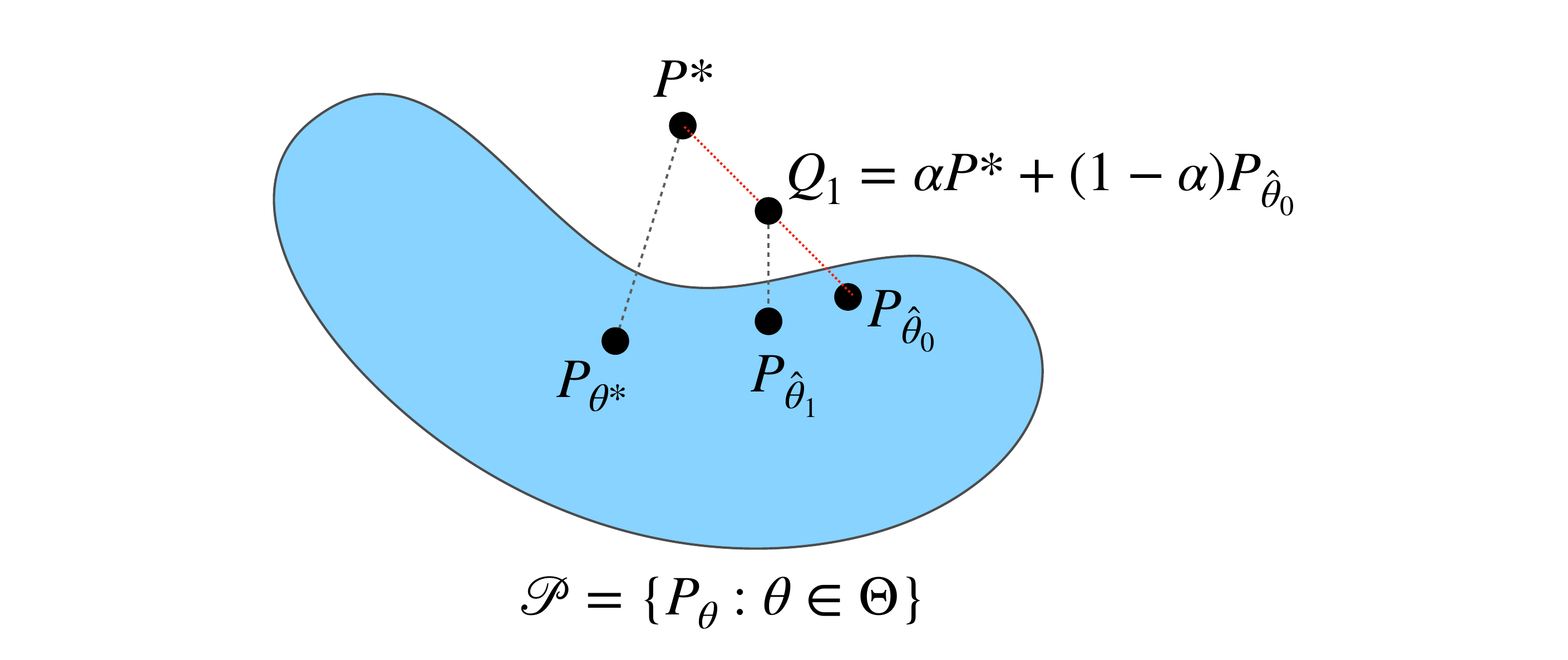}
    \caption{Visual Depiction of Theorem \ref{thm:one_step2}: Fresh information can improve estimation in the context of using MLE given a model class $\cP$, where the MLE can be seen as estimating the KL-projection of the $Q_1$ on the model space. $P_{\theta^*}$ is the projection of the true data distribution $P^*$.}
    \label{fig:likelihood}
\end{figure}

\subsection{An illustrative example}\label{app:examples-one_step}
We first illustrate an example, which captures the main intuition of the two theorems in the section. Consider a simple setting with $\cP=\{N(\mu, 1): \mu\in \bbR\}$ with negative log-likelihood loss. From an empirical risk minimization perspective, this is equivalent to a loss $\ell(\theta,X)=(\theta-X)^2$ -- however, it is better to think about it from a distribution estimation framework since typically the estimator must result in a generative model, to be used in the next stage. Given $X_1,\dots,X_n$ at first stage with sample mean $\tilde{\theta}_0=\bar{X}$, let $Y_1,\dots,Y_m\sim P_{\tilde{\theta}_0}$. With accumulation, the dataset is $D_1$ consisting of all $X$ and $Y$. The optimal estimator based on $D_1$, under the assumption that the data fully comes from a single unknown normal distribution, is $\hat{\theta}_1:=(n\bar{X}+m\bar{Y})/(n+m)$. Note that $D_1|D_0$ is free of $P^*$. Further,
$$\hat{\theta}_0 := \bbE[\hat{\theta}_1|D_0] = \frac{n}{n+m}\bar{X} + \frac{m}{n+m}\hat{\theta}_0 = \bar{X}$$
and the result says that $\hat{\theta}_0=\bar{X}$ (which in this case is the same as $\tilde{\theta}_0$) is as good as $\hat{\theta}_1$. We further note that both estimators are unbiased, but variance of $\hat{\theta}_0$ is lower than that of $\hat{\theta}_1$ -- thereby, $\hat{\theta}_0$ (based only on $D_0$) is a better estimator with respect to this loss. In other words, the additional next stage $D_1$ does not help in any way. Interestingly, the sample size $m$ does not matter.

On the other hand, say we start with $\theta_0\neq \theta^*$. Population distribution at next stage under a mixture setting is $Q_1=\alpha N(\mu^*,1) + (1-\alpha) N(\theta_0, 1)$ -- note that this is not the same as $N(\alpha\mu^*+(1-\alpha)\theta_0, 1)$, in fact, $Q_1\notin \cP$. However, under the negative log-likelihood loss, the KL projection of $Q_1$ on the model space is $N(\alpha\mu^*+(1-\alpha)\theta_0, 1)$. Thus, $\theta_1 = \alpha\mu^*+(1-\alpha)\theta_0$. If $\alpha>0$, $\theta_1$ has a strictly better risk than $\theta_0$ with respect to the true distribution $P^*=N(\mu^*, 1)$. At a sample level, say $\hat{\theta}_1^m$ is based on a finite sample of size $m$ from $Q_1$. Then, we know that $\hat{\theta}_1^m\to \theta_1$ as $m\to\infty$, thereby, with sufficiently large $m$, the risk with $\hat{\theta}_1^m$ is better than $\theta_0$. The precise size depends on $\alpha$ as well.

Lastly, we give an example, where without the condition $\cP$ being convex, one might not improve for low $\alpha$. Consider $\cP=\{N(0,1)\} \cup \{N(1,1)\}$, and the true distribution be $N(0,1)$. If $\theta_0=1$ (i.e., starting at $N(1,1)$), then at next stage $Q_1=\alpha N(0,1)+(1-\alpha)N(1,1)$, whose projection is still $N(1,1)$ if 
$$|(1-\alpha) - 1| < |(1-\alpha) - 0|\iff \alpha < 1/2.$$
Thus, in this case, with low $\alpha$ ($\alpha<1/2$), the estimator is stuck at the sub-optimal $\theta_0$. 

\begin{proof}[Proof of Theorem \ref{thm:one_step1}]
    Let $\hat{\theta}_1$ be an estimator based on $D_1$. In that case, since $D_1|D_0$ is independent of $P^*$, $\bbE[\hat{\theta}_1|D_0]$ is a statistic (free of $P^*$), that is a measureable function of $D_0$ and since $D_1|D_0$ is known, this can be computed. Let $\hat{\theta}_0=\bbE[\hat{\theta}_1|D_0]$. Then by Jensen's inequality, it follows that
    $$\ell(\hat{\theta}_0, X) = \ell\left(\bbE[\hat{\theta}_1|D_0], X\right)\leq \bbE [\ell(\hat{\theta}_1, X) | D_0].$$
    Taking expectation on both sides give the desired inequality. 
\end{proof}

\begin{proof}[Proof of Theorem \ref{thm:one_step2}]
    Firstly, note that $\cR(\theta, P_{\theta'}) = \bbE_{X\sim p_{\theta'}}[-\log p_{\theta}(X)] = \KL(P_{\theta'}\Vert P_{\theta}) + H(P_{\theta'})$, where $H$ is the entropy. Thus, for fixed $\theta'=\theta_0$, the risk $\cR(\theta,P_{\theta_0})$ is minimized at $P_{\theta_0}$, i.e., $\cR(\theta_0,P_{\theta_0})\leq \cR(\theta_1,P_{\theta_0})$. For the first part, if $\theta_1=\theta_0$, there is nothing to show. Otherwise, note that by the definition of $\theta_1$
    \begin{align*}
    &\cR(\theta_1, Q_1) \leq \cR(\theta_0, Q_1)\\
    \Rightarrow & \ \alpha\cR(\theta_1,P^*)+ (1-\alpha)\cR(\theta_1,P_{\theta_0}) \\ \leq & \ \alpha\cR(\theta_0,P^*)+ (1-\alpha)\cR(\theta_0,P_{\theta_0}).\end{align*}
    Now, using the fact that $\cR(\theta_0,P_{\theta_0}) \leq \cR(\theta_1, P_{\theta_0})$, the above inequality shows that if $\alpha>0$, then $\cR(\theta_1, P^*) \leq \cR(\theta_0, P^*)$.
\end{proof}

\section{Data Experiments}

All codes and experiments are on our Github page and will be made public soon. 

\subsection{Discussion on multinomial experiment}
We highlight the following observations from Figure \ref{fig:multinomial experiments}:

\begin{enumerate}
    \item Focusing on the top row (without accumulation), when $\alpha_t$ is fixed, any increasing sequence $n_t$ leads to consistent estimation, as evident from $R_t\to 0$. When $\alpha_t$ slowly decreases (second and third columns), the choice $n_t\asymp t^2$ (or $n_t\asymp t^3$) leads to consistent estimation, as suggested by our theoretical insights. Note that for $\alpha_t\asymp 1/\sqrt{t}$, it seems even the choice $n_t\asymp t$ leads to $R_t\to 0$. This suggests that while the condition in Proposition \ref{prop:decay} is sufficient, it might not be necessary.
    \item For the rapidly decaying $\alpha_t\asymp 1/t^2$ (the right-most column), $R_t$ does not go to 0 for any of the choices of $n_t$ considered here. However, we see that a faster rate $n_t\asymp t^3$ performs the best, and can get the limit $R_{\infty}$ to be less than $R_0$. Although we do not achieve consistency in this case, iterative training is not bad in the sense that the model eventually performs better than the benchmark $R_0$ attained by training on samples coming exclusively from the ground-truth distribution.
    \item Training with accumulation also has a similar trend. The major difference is visible when $n_t$ is constant. Without accumulation, $R_t$ becomes worse over iterations (similar to total collapse), while if the models are trained with accumulation, then it performs much better. For increasing $n_t$ settings, the improvement is not so clear -- it seems for such cases, training the model without accumulation might be better. For example, consider the third column, where $\alpha_t\asymp 1/t$. In this case, for both $n_t\asymp t^2$ and $n_t\asymp t^3$, the performance is better for the case without accumulation. With accumulation, it seems that $R_t$ does not go to 0.
\end{enumerate}
\subsection{Additional numerical experiments}
\label{app: additional experiments}
We consider the following four model classes, where the superscript $*$ indicates the true data distribution and the subscript $t$ indicates the parameter estimated at iteration $t$:
\begin{enumerate}
    \item Statistical Language Model in terms of Categorical Distribution: We set a fixed distribution over $K=20$ classes, parametrized by $\theta^*\in\Delta^{K-1}$. We evaluate the fitted model in terms of $R_t=\bbE\norm{\hat{\theta}_t-\theta^*}_2^2$.
    \item Gaussian model: We choose a fixed univariate Gaussian distribution $N(\mu_*,\sigma_*^2)$ as the true data distribution. We evaluate the fitted model in terms of $W_t=W_2^2(N(\hat{\mu}_t,\hat{\sigma}_t^2), N(\mu_*,\sigma_*^2))$.
    \item Gaussian Mixture Model (GMM): We choose a fixed univariate GMM with $K=5$ components. Associating the mixing measure $G=\sum_k pi_k\delta_{(\mu_k,\sigma_k^2)}$ with the GMM $\sum_k \pi_k N(\mu_k,\sigma_k^2)$, the performance of the fitted models is evaluated using the 1-Wasserstein distance between the mixing measures, i.e., $W_t=W_1(\hat{G}_t, G^*)$.
    \item Logistic Regression: We consider a simple binary classification task with features $X\sim N(0,I_2)$ (iid across all iterations and within an iteration). The true data generating model is $\log p(x)/(1-p(x)) = X^\top\theta^*$, $\theta^*\in\bbR^3$ (where we include an intercept term) and $p(x)=P(y=1|X=x)$. The fitted model is evaluated based on $E_t=\bbE\norm{\hat{\theta}_t - \theta^*}_2^2$.
\end{enumerate}

\begin{figure*}
    \centering
    \includegraphics[width=0.24\linewidth]{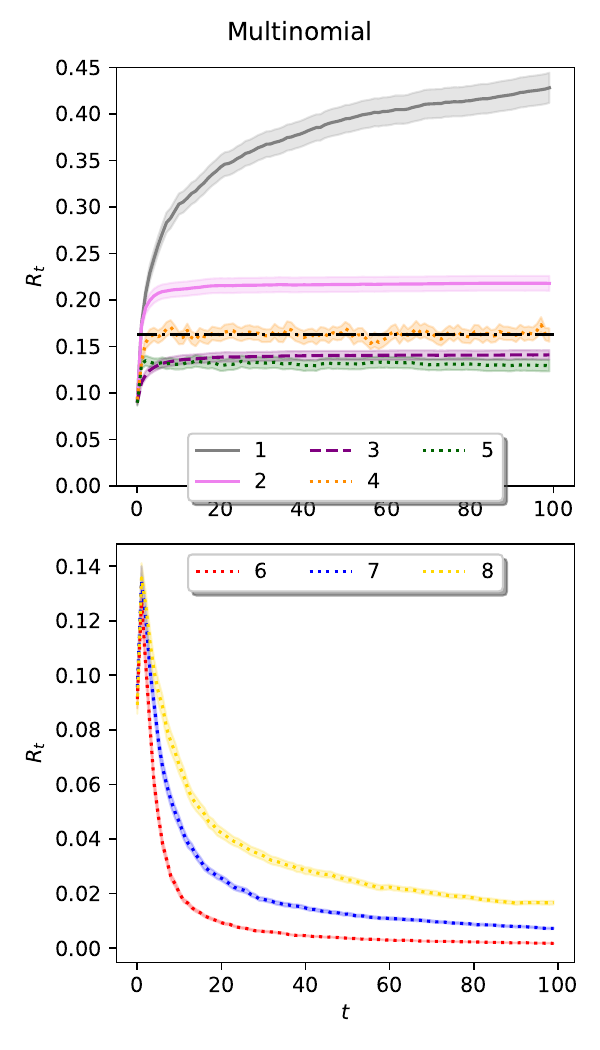}
    \includegraphics[width=0.24\linewidth]{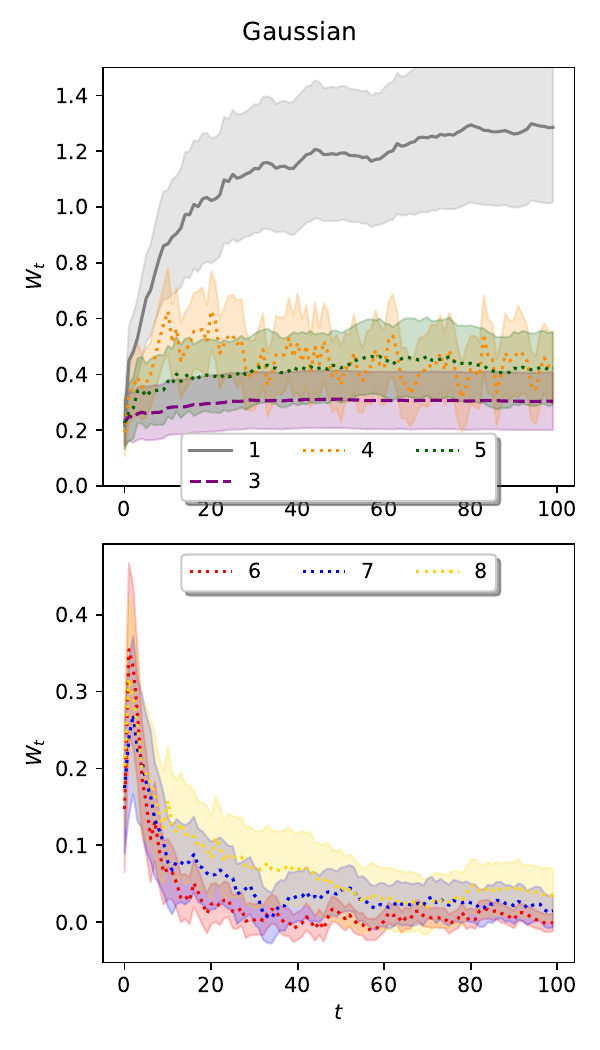}
    \includegraphics[width=0.24\linewidth]{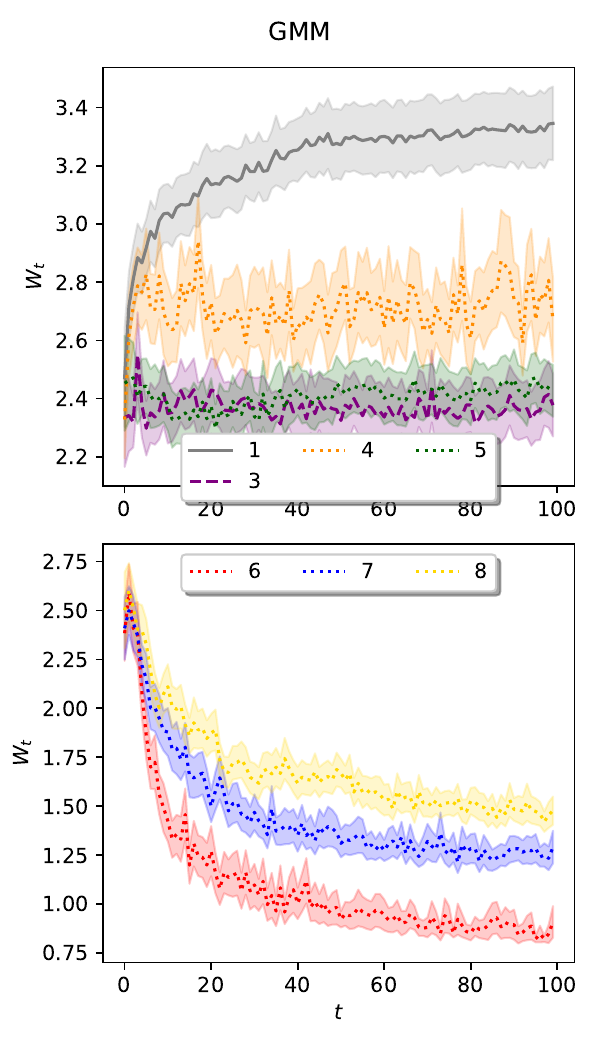}
    \includegraphics[width=0.24\linewidth]{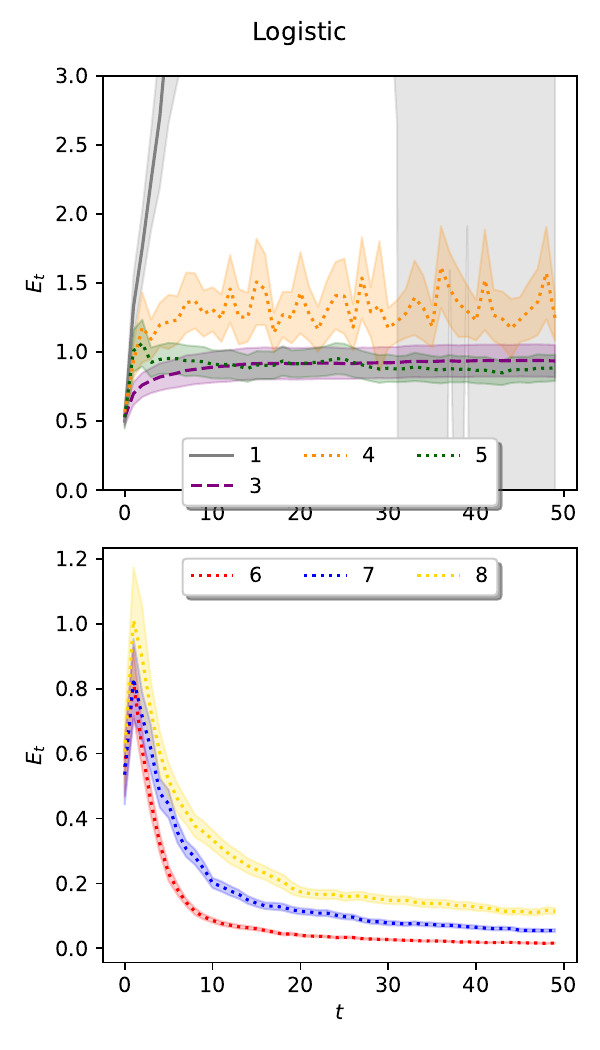}
    \caption{Simulation Studies: Iterative Evolution of iteratively  Trained Statistical Models -- Columns are model classes (categorical, Gaussian, GMM and logistic left to right) and the colors represent the different settings described above in the text. The two rows differentiate between settings where there is partial or total collapse (top row) and statistical consistency over iterations (bottom row)}
    \label{fig:simulation1}
\end{figure*}

\paragraph{Generic Simulation Details:} For each of the above settings, we consider iterative training up to $T$ rounds, where $T=100$ for the first three settings and $T=50$ for logistic regression. Each trajectory for each setting (combination of model from above and the synthetic/fresh data generation scheme below) is repeated $500$ times and the error plots are constructed along with a $95\%$ confidence interval. We consider 8 different settings based on how data is generated at each iteration. For each, at the first iteration the dataset is generated iid from the true data distribution model. Apart from setting 3, we do not accumulate data in the other settings.
\begin{enumerate}
    \item Setting 1: This is the pure synthetic data regime with linearly increasing sample size -- $n_t=n_0 t$ (where $n_0=10$ for categorical and Gaussian, $n_0=50$ for GMM and $n_0=80$ for logistic regression). In this setting $\alpha_t=0$.
    \item Setting 2: This is also pure synthetic data regime with quadratically increasing sample size $n_t=n_0 t^2$ -- we consider this setting only for the categorical model owing to large computational time for other classes due to the huge amount of data.
    \item Setting 3: This is the accumulation setting -- we accumulate all the data from previous iterations for the current time point, generating $n_0$ new synthetic samples from the latest fitted model, thereby making $n_t=n_0 t$. Here again, $\alpha_t=0$.
    \item Setting 4: Fixed sample size and fixed fresh proportion -- Use $n_0$ samples at each iteration drawn from a mixture model with $\alpha_t=\alpha_0$ fixed. We take $\alpha_0=0.3$.
    \item Setting 5: Linearly increasing sample size $n_t=n_0 t$ and decreasing truth proportion $\alpha_t=\alpha_0/t$.
    \item Setting 6: Linearly increasing sample size $n_t=n_0 t$ and fixed truth proportion $\alpha_t=\alpha_0$, so that $n_t\alpha_t=\Omega(t)$.
    \item Setting 7: Linearly increasing sample size $n_t=n_0 t$ and decreasing truth proportion $\alpha_t=\alpha_0/t^{1/2}$ so that $n_t\alpha_t =\Omega(t^{1/2})$.
    \item Setting 8: Linearly increasing sample size $n_t=n_0 t$ and decreasing truth proportion $\alpha_t=\alpha_0/t^{2/3}$ so that $n_t\alpha_t =\Omega(t^{1/3})$.
\end{enumerate}

The results from the experiments are shown in Figure \ref{fig:simulation1}. The left column shows the result for the categorical case, which we analyzed completely in Section \ref{sec:language model} and complements the experiments in Section \ref{sec:simulation}. The other columns show the results for the other models. Although the theoretical results in this work do not directly shed light on the behavior of these iterative evolution, we see that empirically, they behave very similar to the multinomial case, suggesting a wider applicability of our results. Firstly, in the pure synthetic data regime, we see that linearly growing $n_t$ leads to total collapse, while quadratic increase (setting 2) leads to collapse, but the error is bounded. This is similar to the accumulation setting 3. In the case of fresh data available, we see settings 4 and 5 in the upper row, where $n_t\alpha_t$ is fixed (the limit in setting 4, given by our analysis in the previous section, is shown as the dashed horizontal line and matches the results). The last three settings, corresponding to $n_t\alpha_t\uparrow \infty$, shown in the bottom row, all demonstrate that iterative training improves estimation even with synthetic data. Interestingly, in each model last 3 settings, the error increases a bit at the start and then starts decreasing.

\subsection{Additional details for GPT-2 experiment}\label{app: gpt2 details}
Here we describe the particular details of the network that we use for the simulations. At each iteration, we train the whole network from scratch (not just finetuning). We use \texttt{n\_layer}=4 layers with \texttt{n\_heads}=4 attention heads per layer, embedding size of \texttt{embd}=256, with context size \texttt{block\_size}=128 and trained with Adam (learning rate $3\times 10^{-4}$) optimizer using batch size \texttt{batch\_size}=32. Overall, the model has about 16 million parameters.

\paragraph{Simulation settings:} We consider 4 settings for these experiments and repeat each 10 times (each repetition involve randomly permuting the documents before the start of the train-validation split and the iterative process) over $T=10$ iterative rounds. Despite running on a GPU (28 nodes: 8x NVIDIA A40 48GB), one complete experiment took around 30-40 minutes, where surprisingly generating from the model took slightly more time than training (over 3000 epochs). The 3000 epochs was used not just for time constraints, but we found that going beyond leads to sever over-fitting, where the training loss went down but the validation loss kept increasing.

\section{Filtering via a real-vs-synthetic classifier}
\label{app:classifier}
At stage $t$, let $\cC_t$ be a classifier trained to distinguish real samples $Y\sim P^*$ from synthetic samples $Y\sim \widehat P_t$ (the latest model). We assume $\cC_t$ is applied to each observation in the next-stage training batch $Y^{t+1}=\{Y^{t+1}_1,\dots,Y^{t+1}_{n_{t+1}}\}$, and we retain only those classified as real, i.e. $\cC_t(Y^{t+1}_i)= 1$. Denote the retained (filtered) batch by $\tilde Y^{t+1}$, with random size $\tilde n_{t+1}:=|\tilde Y^{t+1}|$.

Recall that we assume the classifier $\cC_t$ has type-I error rate $e_{1t}$ and type-II error rate $e_{2t}$ with the null corresponding to $P^*$ and alternate corresponding to $\hat{P}_t$ , i.e.
\[
e_{1t}:=\mathbb{P}\big(\cC_t(Y)=0\mid Y\sim P^*\big),
\qquad
e_{2t}:=\mathbb{P}\big(\cC_t(Y)=1\mid Y\sim \widehat P_t\big).
\]
Then each point is retained with probability
\[
q_{t+1}
:=\mathbb{P}\big(\cC_t(Y)=1\big)
=\alpha_{t+1}(1-e_{1t})+(1-\alpha_{t+1})e_{2t},
\]
so $\tilde n_{t+1}\sim\mathrm{Binomial}(n_{t+1},q_{t+1})$ and hence
\[
\mathbb{E}[\tilde n_{t+1}]
=
q_{t+1}\,n_{t+1}
=
\big[\alpha_{t+1}(1-e_{1t})+(1-\alpha_{t+1})e_{2t}\big]\,n_{t+1}.
\]
Moreover, the filtered data $\tilde Y^{t+1}$ are i.i.d. from a mixture distribution
\[
\tilde P_{t+1}
=
\tilde\alpha_{t+1} P^* + (1-\tilde\alpha_{t+1})\widehat P_t,
\qquad
\tilde\alpha_{t+1}
=
\frac{\alpha_{t+1}(1-e_{1t})}{\alpha_{t+1}(1-e_{1t})+(1-\alpha_{t+1})e_{2t}}.
\]

\medskip
\noindent
Define the filtered estimator $\tilde\theta_{t}$ as the empirical class-frequency vector based on the filtered data $\tilde Y^{t}$:
\[
\tilde\theta_t(k)
:=
\frac{1}{\tilde n_t}\sum_{j=1}^{\tilde n_t}\mathbf 1\{\tilde Y^t_j=k\},
\qquad k\in\{1,\dots,K\},
\]
(with the convention that the procedure resamples until $\tilde n_t\ge1$, or equivalently all statements below are conditional on $\{\tilde n_t\ge1\}$).
We measure performance via the filtered risk $R_t^{f}:=\mathbb{E}\big\|\tilde\theta_t-\theta^*\big\|_2^2$, where the expectation is taken over randomness of both data $Y^t$ and induced by the filtering/classifier decisions (equivalently the induced random retained set / retained sample size).

\begin{assumption}[Label-agnostic filtering]\label{ass:label-agnostic}
For each stage $t\ge 0$, the classifier $\cC_t$ is applied to each observation $Y$ drawn from the mixture
$\alpha_{t+1} P^* + (1-\alpha_{t+1})\widehat P_t.$ Conditional on whether $Y$ is drawn from $P^*$ or from $\widehat P_t$, the classification decision of $\cC_t$ is independent of the categorical label of $Y$. Equivalently, for any class $k$,
$$\mathbb{P}\!\big(\cC_t(Y)=1\mid Y\sim P^*,\,Y=k\big)
=\mathbb{P}\!\big(\cC_t(Y)=1\mid Y\sim P^*\big), \ \mathbb{P}\!\big(\cC_t(Y)=1\mid Y\sim \widehat P_t,\,Y=k\big)
=
\mathbb{P}\!\big(\cC_t(Y)=1\mid Y\sim \widehat P_t\big),
$$
where these probabilities equal $1-e_{1t}$ and $e_{2t}$, respectively.
\end{assumption}

\begin{theorem}[Filtered-risk recurrence]\label{thm:filtered-risk-recurrence}
Let $\tilde R_t:=\mathbb{E}\|\tilde\theta_t-\theta^*\|_2^2$. Under the filtration model described above satisfying Assumption~\ref{ass:label-agnostic}, for each $t\ge 1$,
\begin{equation}\label{eq:tildeRt-recurrence}
\tilde R_t
=
\mathbb{E}\!\Big[\frac{1}{\tilde n_t}\Big]\,(1-\|\theta^*\|_2^2)
+
\Big(1-\mathbb{E}\!\Big[\frac{1}{\tilde n_t}\Big]\Big)\,(1-\tilde\alpha_t)^2\,\tilde R_{t-1},
\end{equation}
where $\mathbb{E}[1/\tilde n_t]$ is understood under the same convention used to define $\tilde\theta_t$ (i.e. conditioning on $\{\tilde n_t\ge 1\}$).
\end{theorem}

\begin{proof}
Fix $t\ge 1$ and write
\[
\rho_t:=\tilde\alpha_t\,\theta^*+(1-\tilde\alpha_t)\,\tilde\theta_{t-1}.
\]
Conditional on $(\tilde\theta_{t-1},\tilde n_t)$ (with $\tilde n_t\ge 1$), the filtered batch $\tilde Y^t$ is i.i.d. $\mathrm{Cat}(\rho_t)$ and $\tilde\theta_t$ is the empirical frequency vector based on $\tilde n_t$ samples. Hence,
\begin{equation}\label{eq:cond-bias-var}
\mathbb{E}\!\left[\|\tilde\theta_t-\theta^*\|_2^2 \,\big|\, \tilde\theta_{t-1},\tilde n_t\right]
=
\underbrace{\mathbb{E}\!\left[\|\tilde\theta_t-\rho_t\|_2^2 \,\big|\, \tilde\theta_{t-1},\tilde n_t\right]}_{\text{variance}}
+
\underbrace{\|\rho_t-\theta^*\|_2^2}_{\text{bias}^2}.
\end{equation}
The two terms are explicit:
\[
\mathbb{E}\!\left[\|\tilde\theta_t-\rho_t\|_2^2 \,\big|\, \tilde\theta_{t-1},\tilde n_t\right]
=
\frac{1-\|\rho_t\|_2^2}{\tilde n_t},
\qquad
\|\rho_t-\theta^*\|_2^2
=
(1-\tilde\alpha_t)^2\|\tilde\theta_{t-1}-\theta^*\|_2^2.
\]
Therefore,
\begin{equation}\label{eq:cond-step}
\mathbb{E}\!\left[\|\tilde\theta_t-\theta^*\|_2^2 \,\big|\, \tilde\theta_{t-1},\tilde n_t\right]
=
\frac{1-\|\rho_t\|_2^2}{\tilde n_t}
+
(1-\tilde\alpha_t)^2\|\tilde\theta_{t-1}-\theta^*\|_2^2.
\end{equation}

\noindent\emph{Unbiasedness of $\tilde\theta_t$.} We first verify directly that $\mathbb{E}\tilde\theta_1=\theta^*$. Conditional on $\widehat\theta_0$, the filtered sample $\tilde Y^1$ is i.i.d. from
\[
\tilde P_1=\tilde\alpha_1 P^* + (1-\tilde\alpha_1)\widehat P_0
=\mathrm{Cat}\!\Big(\tilde\alpha_1\theta^*+(1-\tilde\alpha_1)\widehat\theta_0\Big),
\]
so, for each coordinate $k$ (and conditioning on $\tilde n_1\ge 1$),
\[
\mathbb{E}\big[\tilde\theta_1(k)\mid \widehat\theta_0\big]
=
\tilde\alpha_1\theta^*(k)+(1-\tilde\alpha_1)\widehat\theta_0(k).
\]
Taking expectation and using $\mathbb{E}\widehat\theta_0=\theta^*$ yields
\[
\mathbb{E}\tilde\theta_1(k)
=
\tilde\alpha_1\theta^*(k)+(1-\tilde\alpha_1)\mathbb{E}\widehat\theta_0(k)
=
\theta^*(k).
\]
Thus $\mathbb{E}\tilde\theta_1=\theta^*$.
For the induction step, assume $\mathbb{E}\tilde\theta_{t-1}=\theta^*$. By the same reasoning,
\[
\mathbb{E}\big[\tilde\theta_t \mid \tilde\theta_{t-1}\big]
=
\tilde\alpha_t\theta^*+(1-\tilde\alpha_t)\tilde\theta_{t-1},
\]
so taking expectations gives
\[
\mathbb{E}\tilde\theta_t
=
\tilde\alpha_t\theta^*+(1-\tilde\alpha_t)\mathbb{E}\tilde\theta_{t-1}
=
\theta^*.
\]

Next, by the label-agnostic filtering assumption, the random count $\tilde n_t$ depends only on the keep probabilities (equivalently on $(\alpha_t,e_{1,t-1},e_{2,t-1})$) and not on the categorical labels; in particular, $\tilde n_t$ is independent of $\tilde\theta_{t-1}$, and hence independent of $\rho_t$ (which is a function of $\tilde\theta_{t-1}$). Taking expectation of \eqref{eq:cond-step} and using this independence yields
\begin{equation}\label{eq:uncond-step}
\tilde R_t
=
(1-\tilde\alpha_t)^2\tilde R_{t-1}
+
\mathbb{E}\!\Big[\frac{1}{\tilde n_t}\Big]\,
\mathbb{E}\!\big[1-\|\rho_t\|_2^2\big].
\end{equation}

Finally, expand $\|\rho_t\|_2^2$ and use $\mathbb{E}\tilde\theta_{t-1}=\theta^*$:
\[
\|\rho_t\|_2^2
=
\|\theta^*\|_2^2
+(1-\tilde\alpha_t)^2\|\tilde\theta_{t-1}-\theta^*\|_2^2
+2(1-\tilde\alpha_t)\theta^{*\top}(\tilde\theta_{t-1}-\theta^*),
\]
so after taking expectations the cross term vanishes and
\[
\mathbb{E}\|\rho_t\|_2^2
=
\|\theta^*\|_2^2+(1-\tilde\alpha_t)^2\tilde R_{t-1}.
\]
Hence
\[
\mathbb{E}\big[1-\|\rho_t\|_2^2\big]
=
1-\|\theta^*\|_2^2-(1-\tilde\alpha_t)^2\tilde R_{t-1}.
\]
Substitute this into \eqref{eq:uncond-step} and rearrange to obtain \eqref{eq:tildeRt-recurrence}:
\[
\tilde R_t
=
(1-\tilde\alpha_t)^2\tilde R_{t-1}
+
\mathbb{E}\!\Big[\frac{1}{\tilde n_t}\Big]\Big(1-\|\theta^*\|_2^2-(1-\tilde\alpha_t)^2\tilde R_{t-1}\Big)
=
\mathbb{E}\!\Big[\frac{1}{\tilde n_t}\Big](1-\|\theta^*\|_2^2)
+
\Big(1-\mathbb{E}\!\Big[\frac{1}{\tilde n_t}\Big]\Big)(1-\tilde\alpha_t)^2\tilde R_{t-1}.
\]
\end{proof}

\begin{remark}[Mean-field effective-sample-size approximation]\label{rem:mf}
A convenient closed-form approximation replaces $\tilde n_{t+1}$ by its mean $\mathbb{E}\tilde n_{t+1}=q_{t+1}n_{t+1}$, yielding
\[
\tilde{R}_{t+1}
\approx
(1-\tilde\alpha_{t+1})^2\tilde{R}_{t}
+
\frac{1-\|\theta^*\|_2^2-(1-\tilde\alpha_{t+1})^2\tilde{R}_{t}}{q_{t+1}n_{t+1}},
\qquad
q_{t+1}=\alpha_{t+1}(1-e_{1t})+(1-\alpha_{t+1})e_{2t}.
\]
This highlights the tradeoff induced by filtering: increasing $\tilde\alpha_{t+1}$ improves the ``fresh information'' fraction but decreases the effective sample size through $q_{t+1}$.
\end{remark}

\section{ Naturally Adaptive  Mixture Weights $\alpha_t$}
\label{app:adaptive}
Here we present details and additional examples of adaptive schemes of mixture weights $\alpha_t$, hinted in our discussion in Section~\ref{sec:dis}. 

For example, if the  mixture weights $\alpha_t$ follow either of the following adaptive conditions for all $t\geq 1$:
\begin{enumerate}
    \item[(i)] \((1 - \alpha_t)^2 \leq \text{min}\cbracket{1,\frac{1}{\lambda R_{t-1}}}\), for some constant \(\lambda > 0\);
    \item[(ii)] \((1 - \alpha_t)^2 \leq \exp(-\lambda R_{t-1})\), for some \(\lambda > 1\).
\end{enumerate}
then model collapse is avoided. Moreover, optimistic upper-bounds for the limiting risk can be obtained, suggesting natural improvement and limiting consistency of the iterative training process.

Bellow we provide the general result from which the above remarks follow.
\begin{proposition}
Let $g_{\lambda}(.):[0,2]\to(0,1]$ be any decreasing function with constant $\lambda>1$, such that for all $x\in[0,2]$, $g_{\lambda}(x)<\frac{1}{\lambda x}$.
Then collapse is avoided if $(1 - \alpha_t)^2\leq g_{\lambda}(R_{t-1})$ for all $t\geq1$ and the following holds:\begin{enumerate}
    \item For fixed sample size $n_t=n$, $$R_t \leq R_0 + \frac{n-1}{\lambda n};$$
    \item For any sequence $n_t \uparrow \infty$, $$R_{\infty}<\frac{1}{\lambda}.$$
\end{enumerate}
\end{proposition}

\begin{proof}
 Recall \begin{align*}
      R_t &= \frac{n_{0}R_{0}}{n_t} + \frac{n_t-1}{n_t} (1-\alpha_t)^2 R_{t-1} \\
      &\leq \frac{n_{0}R_{0}}{n_t} + \frac{n_t-1}{n_t} g_{\lambda}(R_{t-1}) R_{t-1} \\
      &< \frac{n_{0}R_{0}}{n_t} + \frac{n_t-1}{n_t} \cdot \frac{1}{\lambda} \rbracket{\text{as } g_{\lambda}(x)<\frac{1}{\lambda x}}.
 \end{align*}  Replacing $n_t=n$, we directly get  $$R_t \leq R_0 + \frac{n-1}{\lambda n},$$ and for $n_t \uparrow \infty$, the first term vanishes and the ratio $\frac{n_t-1}{n_t}$ converges to 1 giving $$R_{\infty}<\frac{1}{\lambda}.$$
\end{proof}

\end{document}

%% file: macros.tex
\newcommand{\et}
{\widehat{\theta}_{t}}
\newcommand{\ett}
{\widehat{\theta}_{t+1}}

\newcommand{\at}
{\alpha_{t}}
\newcommand{\watt}
{\widehat{\alpha}_{t+1}}
\newcommand{\att}
{\alpha_{t+1}}

\newcommand{\nbracket}[1]{\left( #1 \right)}
\newcommand{\cbracket}[1]{\left\{ #1 \right\}}
\newcommand{\rbracket}[1]{\left[ #1 \right]}

\def\argmin{{\arg\min}}

\def\bbE{\mathbb{E}}

\def\bbR{\mathbb{R}}

\def\cC{\mathcal{C}}

\def\cP{\mathcal{P}}

\def\cR{\mathcal{R}}

\def\cU{\mathcal{U}}

\def\iid{{\sf iid}}

\def\KL{{\sf KL}}

\def\norm#1{\left\|#1\right\|}

\def\TV{{\sf TV}}

\def\Var{{\sf Var}}